\documentclass[11pt]{article}
\usepackage[marginratio=1:1,height=600pt,width=460pt,tmargin=100pt]{geometry}
\usepackage{authblk}
\usepackage{mathtools}
\usepackage{microtype}
\usepackage{graphicx}
\usepackage{subcaption}
\usepackage{booktabs} 
\usepackage{makecell}
\usepackage{wrapfig}
\usepackage{diagbox}
\usepackage{tikz}
\usepackage{enumitem}%
\usepackage{amsfonts}
\usepackage{xcolor}

\usepackage{amsthm}
\usepackage{cprotect}
\usepackage{hyperref}
\usepackage{amsmath,amssymb}
\usepackage{multirow}
\usepackage{fancyhdr}
\renewcommand{\headrulewidth}{0pt}
\fancypagestyle{firststyle}
{
   \fancyhf{}
   \fancyfoot[R]{\thepage}
   \renewcommand{\headrulewidth}{0pt} 
}
\usepackage{fancyvrb}
\usepackage{natbib}
\usepackage{algorithm}
\usepackage{algorithmic}
\allowdisplaybreaks

\newcommand{\diag}[1]{\text{diag}(#1)}

\newcommand{\R}{\mathbb{R}}
\newcommand{\p}{\mathbb{P}}

\newcommand*\circled[1]{\tikz[baseline=(char.base)]{
\node[shape=circle,draw,inner sep=0.6pt] (char) {#1};}
}

\newcommand{\SPCI}{\texttt{SPCI}}
\newcommand{\MultiSPCI}{\texttt{MultiDimSPCI}}
\newcommand{\Ctalpha}{\widehat{C}_{t-1}(X_t,\alpha)}
\newcommand{\hatQt}{\widehat{Q}_t}
\newcommand{\betahat}{\hat{\beta}}
\newcolumntype{P}[1]{>{\centering\arraybackslash}p{#1}}

\newlength{\tempdima}
\newcommand{\rowname}[1]
{\rotatebox{90}{\makebox[\tempdima][c]{\textbf{#1}}}}
\renewcommand{\thesubfigure}{\alph{subfigure}}
\newcommand{\mycaption}[1]
{\refstepcounter{subfigure}\textbf{(\thesubfigure) }{\ignorespaces #1}}

\theoremstyle{plain}
\newtheorem{theorem}{Theorem}[section]

\newtheorem{lemma}[theorem]{Lemma}
\newtheorem{corollary}[theorem]{Corollary}
\theoremstyle{definition}
\newtheorem{definition}[theorem]{Definition}
\newtheorem{assumption}[theorem]{Assumption}
\theoremstyle{remark}
\newtheorem{remark}[theorem]{Remark}

\usepackage{setspace}
\begin{document}
\title{Conformal prediction for multi-dimensional time series by ellipsoidal sets}

\author[1]{Chen Xu$^*$}
\author[1]{Hanyang Jiang\thanks{Equal contribution}}
\author[1]{Yao Xie\thanks{Correspondence: yao.xie@isye.gatech.edu}}
\affil[1]{{\small H. Milton Stewart School of Industrial and Systems Engineering, Georgia Institute of Technology.}}

\maketitle
\thispagestyle{firststyle}
\vspace{-0.3in}
\begin{abstract}
Conformal prediction (CP) has been a popular method for uncertainty quantification because it is distribution-free, model-agnostic, and theoretically sound. For forecasting problems in supervised learning, most CP methods focus on building prediction intervals for univariate responses. In this work, we develop a sequential CP method called \MultiSPCI{} that builds prediction \textit{regions} for a multivariate response, especially in the context of multivariate time series, which are not exchangeable. Theoretically, we estimate \textit{finite-sample} high-probability bounds on the conditional coverage gap. Empirically, we demonstrate that \MultiSPCI{} maintains valid coverage on a wide range of multivariate time series while producing smaller prediction regions than CP and non-CP baselines.
\end{abstract}

\section{Introduction}

Conformal prediction (CP) has been a popular distribution-free technique to quantify uncertainty in modern machine learning \citep{vovkinductive}. In building predictive algorithms, CP can enhance trained machine learning estimators to output not just point estimates but also provide uncertainty sets that contain the unobserved ground truth with user-specified high probability. As a result, CP has been applied successfully to many applications, such as anomaly detection \citep{xu2021ECAD}, classification \citep{MJ_classification,xuERAPS2022}, regression \citep{jackknife+}, and so on. In a nutshell, CP methods work as wrappers that take in three components: a black-box predictive model $f$, an input feature $X$, and a potential output $Y$. Then, it designs a so-called ``non-conformity'' score that measures how \textit{non-conforming} the potential output is. Naturally, the uncertainty set conditioning on the input feature and the predictor model would contain all potential outputs that are \textit{conforming} to the past.

Most successful applications of CP have considered $Y$ as an univariate variable. In the regression setting \citep{j+ab}, CP methods thus output \textit{prediction intervals} while in the classification setting \citep{Candes_classification}, these methods produce \textit{prediction sets}. With mild assumptions, such as assuming that data $(X,Y)$ are exchangeable, these one-dimensional uncertainty sets can theoretically guarantee high coverage probability. Recent works have also extended such guarantees to non-exchangeable observations and either quantify the coverage gap in finite training samples \citep{barber2023conformal} or show asymptotic coverage convergence \citep{xu2023sequential}.

Despite the success of CP on scalar outputs $Y$, effective use of CP on multi-dimensional outputs is considerably less studied, \textit{especially} when data are non-exchangeable as in the case of multivariate time-series forecasting. Moreover, there can be complex dependence between the multiple dimensions of the time series, making the problem more interesting and important. Specifically, the goal is not just to provide a prediction interval for each dimension of $Y$ but to produce an uncertainty region that captures the correlation within $Y$ and jointly contains all coordinates of $Y$. While uncertainty quantification methods for this problem have existed outside CP, as in vector auto-regressive models \citep{salinas2020deepar}, these approaches often have strong modeling or data assumption and lack rigorous theoretical justifications. On the other hand, various multi-dimensional CP methods have been proposed. Yet, they are either repeated use of one-dimensional CP methods \citep{stan021conformal} or fail to work beyond exchangeability \citep{messoudi2022ellipsoidal}.

We highlight the differences against copula-based CP methods, which have been developed for multi-dimensional forecasting. The initial approach developed in \citep{messoudi2021copula} assumed exchangeability, which is unsuitable for time series. The recent development by \citep{sun2024copula} proposes copula CP for multi-step time-series forecasting. However, their theory assumed that each data sample of an entire time series is drawn i.i.d. from an unknown distribution, hence ignoring temporal dependency. We further introduce copula and its use in CP in Section \ref{sec:copula}, with additional comparison against baselines in Section \ref{sec:real_data}.

\begin{table*}[!t]
\vspace{-0.1in}
\caption{A $2\times2$ taxonomy of conformal prediction approaches (not an exhaustive list), categorized based on the dimension of the response variable $Y$ (rows) and data assumptions (columns).}\label{CP_table}
\resizebox{\linewidth}{!}{
\begin{tabular}{c|c|c}
 & Exchangeable & Non-exchangeable \\\hline
Univariate $Y$  & \makecell[c]{\citep{vovkinductive} \\ \citep{jackknife+, j+ab}}  & \makecell[c]{\citep{Zaffran2022AdaptiveCP,xu2023conformal} \\ \citep{xu2023sequential, barber2023conformal}} \\  \hline
Multivariate $Y$  &  \makecell[c]{\citep{messoudi2021copula,diquigiovanni2022conformal} \\ \citep{johnstone2022exact,feldman2023calibrated}}  &  \makecell[c]{\textbf{Ours} \\ \citep{stan021conformal, sun2024copula}} \\
\end{tabular}}
\vspace{-0.1in}
\end{table*} 
Hence, the central focus of this work is to advance CP in the context of multivariate time-series forecasting. Specifically, we build ellipsoidal prediction sets whose size is adaptively and efficiently calibrated during test time. We also provide rigorous theoretical guarantees and extensive experiments to showcase the utility of the proposed method. In summary, our contributions are
\begin{itemize}[itemsep=0em,topsep=0.5em]
    \item We propose a sequential conformal prediction method that builds ellipsoidal prediction regions for multivariate time series. In particular, sizes of the ellipsoids are sequentially re-estimated during test time to ensure adaptiveness and accuracy.
    \item We provide finite-sample high-probability bounds for the coverage gap of constructed prediction regions, which do not depend on the exchangeability of the observations.
    \item We empirically verify on multivariate time-series (up to dimension 20) that \MultiSPCI{} can yield smaller prediction regions than baseline CP and non-CP methods without losing coverage.
\end{itemize}
For clarity, a taxonomy of existing CP methods is in Table \ref{CP_table} to highlight our role within the CP literature. In this paper, we assume that the noise sequence in the data-generating process (see Eq \eqref{eq:dgp}) is stationary and strongly mixing, where the original data sequence does not need to be exchangeable. Meanwhile, we impose some standard assumptions on the tail behavior of the distribution to derive a non-asymptotic bound on the conditional guarantee. We highlight that our guarantees differ from existing multi-dimensional CP works that assume exchangeability \citep{messoudi2021copula,messoudi2022ellipsoidal} or i.i.d. data \citep{sun2024copula}.

We adopt the standard notations. For a random process $\{X_n\}_n$, $X_n =o_p(a_n)$ means that $|X_n|/a_n\stackrel{p}{\rightarrow}0$. For function $f(n)$ and $g(n)$, $f(n)=\tilde{O}(g(n))$ means that $f(n)=O(g(n)\log(g(n))^k)$ for some $k$. Besides, for event $A$ and $B$, the notation $A\big\vert B$ means $A$ under the condition $B$. For vector $u, v \in \R^p$, $u\otimes v$ is the outer product of $u$ and $v$.

\subsection{Literature review}
\noindent \textit{CP with exchangeable data.} ~The field of CP started in \citep{vovk2005algorithmic} and has been widely used for uncertainty quantification due to its flexibility and robustness. On a high level, we define a ``non-conformity'' score and evaluate such scores on a hold-out calibration set. Then, uncertainty sets include all potential observations whose non-conformity scores are within the empirical quantiles of the calibration scores. Assuming nothing but that input data are exchangeable, CP methods have been successfully developed in different applications \citep{wisniewski20a,xu2021ECAD,xuERAPS2022}, in addition to the active research on proper designs of non-conformity scores \citep{MJ_classification,huang2023conformalized_gnn,gui2023conformalized}. Comprehensive reviews of conformal prediction can be found in \citep{conformalreview,angelopoulos2021gentle}.

\vspace{0.05in}
\noindent \textit{CP for one-dimensional time series.} ~ Two general trends of extending beyond exchangeability work well for univariate $Y$. The first considers adaptively adjusting the significance level $\alpha$ during test time to account for mis-coverage. Such works include \citep{Gibbs2021AdaptiveCI,Zaffran2022AdaptiveCP,lin2022conformal_1}. The recent work \citep{angelopoulos2024conformal} extends such framework through the lens of control theory to prospectively model non-conformity scores in online settings. The second considers weighing the past non-conformity score non-equally so that scores more similar to the present are given higher weights. Such works include \citep{CPcovshift,xu2021conformal,xu2023sequential,xu2023spatio,nair2023randomization}, some of which have successfully been applied to univariate time series. The recent work \citep{barber2023conformal} also suggests that re-weighting can be a general scheme to account for non-exchangeability. Our 
\MultiSPCI{} is similar to the second line of approaches but works in high dimensions.

\vspace{0.05in}
\noindent \textit{CP for multi-dimensional data.} ~ Numerous works have been on this topic. \citep{stan021conformal} builds coordinate-wise prediction intervals for multi-horizon time-series prediction using Bonferroni correction of the significance level. For multivariate functional data, a similar idea of building prediction \textit{bands} was studied in \citep{diquigiovanni2022conformal}, where this idea was further developed for time series \citep{diquigiovanni2021distribution}. In addition, \citep{messoudi2021copula} develops a principled way to determine the length of coordinate-wise intervals by using copula, resulting in hyper-rectangular prediction regions. The extension of copula for time-series forecasting was later studied in \citep{sun2024copula}. 
However, it is important to note that the use of hyper-rectangles can be sub-optimal in many cases, even in the two-dimensional instances when the true conditional distribution $Y|X$ is $N(f(X), \Sigma)$ with a non-zero off-diagonal entry in $\Sigma$. 
To overcome this, \citep{messoudi2022ellipsoidal} considers ellipsoidal uncertainty sets that rely on data exchangeability. A more exact quantification of the uncertainty set is studied in \citep{johnstone2022exact}, which, however, strongly depends on the underlying predictive model of $Y$.
As a result, extending CP for multivariate time-series forecasting beyond using hyper-rectangles still needs to be explored.

\vspace{0.05in}
\noindent \textit{Uncertainty quantification beyond CP.} ~The task of building uncertainty set for unobserved response has been widely studied beyond CP. There has been a long history of using copula to capture the joint distribution of multivariate response by relating the joint cumulative distribution function (CDF) with each marginal CDF \citep{sklar1959fonctions,elidan2013copulas}. Meanwhile, \citep{dobriban2023joint} uses conditional pivots to construct joint coverage regions for parameters and observations, extending beyond CP. However, its utility beyond exchangeable data remains unclear. On the other hand, there has been extensive development in the \textit{probabilistic forecasting} literature, popular examples of which include the DeepAR \citep{salinas2020deepar} and Temporal Fusion Transformer \citep{lim2021temporal}. Such approaches optimize (variants of) the pinball loss to estimate quantiles of multivariate responses but typically require extensive hyper-parameter tuning and return hyper-rectangular uncertainty sets. We will show experimentally that their performances are often worse than their CP counterparts.
Lastly, \citep{feldman2023calibrated} uses CP in the representation space learnt by a deep generative model, allowing general prediction sets for multivariate data. Coverage guarantee for exchangeable data is proved, and extension to time series remains unexplored.

\section{Problem setup}

We consider a multi-dimensional time-series regression setup: for time index $t =1,2,\ldots$, assume observations $Z_t=(X_t, Y_t)$ are sequentially revealed, where $Y_t\in \R^p$ are $p$-dimensional vector variables and $X_t \in \R^d$ are $d$-dimensional features. The features $X_t$ may be the history of $Y_t$ or contain other variables that help predict $Y_t$.
In particular, we allow arbitrarily unknown correlation among the observations $Z_t$. Let the first $T$ samples $\{Z_t\}_{t=1}^T$ be the training data.

Our goal is to sequentially construct prediction \textit{regions} $\Ctalpha$ starting from $t=T+1$, which depends on past observations, the current feature $X_t$, and a user-specified significance level $\alpha \in [0,1]$. In particular, we desire the prediction regions to contain the true observations $Y_t$ with a probability at least $1-\alpha$. Mathematically, there are two types of coverage guarantees to be satisfied by $\Ctalpha$. The first is the weaker \textit{marginal} coverage:
\begin{equation}\label{marginal_cov}
    \mathbb P(Y_t \in \Ctalpha) \geq 1-\alpha, \forall t,
\end{equation}
while the second is the stronger \textit{conditional} coverage:
\begin{equation}\label{cond_cov}
    \mathbb P(Y_t \in \Ctalpha|X_t) \geq 1-\alpha, \forall t.
\end{equation}
If $\Ctalpha$ satisfies \eqref{marginal_cov} or \eqref{cond_cov}, it is called marginally or conditionally valid, respectively. When $\Ctalpha$ satisfies the coverage guarantees, we further construct regions that are as small as possible to quantify uncertainty precisely.

\section{Method}

In this section, we first propose the ellipsoidal uncertainty set that effectively quantifies multi-dimensional prediction. We then discuss several benefits of the proposed approach against alternatives. We finally suggest alternative forms of the uncertainty set beyond using ellipsoids.

\subsection{Ellipsoidal uncertainty set}

We build the prediction regions in the shape of ellipsoids and calibrate the radius of ellipsoids using conformal prediction for univariate time series. Recall that we have access to $T$ training data $Z_t=(X_t,Y_t)$ for $t=1,\ldots,T$. Assume we have been given an algorithm $\hat{f}$, trained on a separate set $I$ to perform point prediction for $Y$. Meanwhile, we collect \textit{prediction} residuals $\hat \varepsilon_t \in \mathbb R^p$
\[\hat \varepsilon_t = Y_t-\hat{f}(X_t),~t=1,\ldots,T.\] This approach is similar to the split conformal prediction \citep{vovkinductive} or leave-one-out techniques \citep{jackknife+, xu2023conformal}.

To define the ellipsoidal uncertainty set, we first denote the covariance estimator over the prediction residuals as 
\begin{equation}\label{eq:global_cov}
    \widehat \Sigma=\frac{1}{T-1}\sum_{t=1}^{T}(\hat \varepsilon_t-\bar{\varepsilon})(\hat \varepsilon_t-\bar{\varepsilon})^T,  
\end{equation}
where $\bar{\varepsilon}=\frac{1}{T}\sum_{t=1}^{T} \hat \varepsilon_t$ is the sample mean vector over the residuals. As the definition of an ellipsoid will rely on the inverse of $\widehat \Sigma$ in \eqref{eq:global_cov}, which may not be invertible, we consider a low-rank approximation of $\widehat \Sigma$ as follows. Let the singular value decomposition of $\widehat{\Sigma}$ be $\widehat{\Sigma}=USV^T$, where $S=\diag{\lambda_1,\ldots,\lambda_p}$ is the diagonal matrix of singular values satisfying $\lambda_1\geq\ldots\geq \lambda_p\geq 0$, and $U$ and $V$ satisfy $UU^T=VV^T=I_p$. Given a small positive threshold $\rho>0$, the low-rank approximation $\widehat{\Sigma}_{\rho}$ is
\begin{equation}\label{eq:cov_est_rank_k}
    \widehat{\Sigma}_{\rho}=U_{\rho}S_{\rho}V_{\rho}^T,
\end{equation}
where $S_{\rho}=\diag{\lambda_1,\ldots,\lambda_k}$ for which $\lambda_k \geq \rho$, and $U_{\rho}$ and $V_{\rho}^T$ contain the first $k$ columns and rows of $U$ and $V^T$, respectively. The pseudo-inverse $\widehat{\Sigma}_{\rho}^{-1}$ is thus written as 
\begin{equation}\label{eq:pseudo_inv}
\widehat{\Sigma}_{\rho}^{-1}=V_{\rho}S_{\rho}^{-1}U_{\rho}^T,
\end{equation} 
where $S_{\rho}^{-1}=\diag{1/\lambda_1,\ldots,1/\lambda_k}$. Using $\widehat{\Sigma}_{\rho}^{-1}$, which is always well-defined, an ellipsoid with radius $r$ can thus be written as 
\[\mathcal B(r, \bar{\varepsilon}, \widehat \Sigma_{\rho})=\{x \in \R^p: (x-\bar{\varepsilon})^T\widehat{\Sigma}_{\rho}^{-1} (x-\bar{\varepsilon}) \leq r\}.\]

We then find an appropriate radius $r$ using time-series conformal prediction methods. First, given a new residual $\hat \varepsilon=Y-\hat{f}(X)$ and the pseudo-inverse $\widehat{\Sigma}_{\rho}^{-1}$ in \eqref{eq:pseudo_inv}, we define the scalar non-conformity score $\hat e(Y)$ as 
\begin{equation}\label{eq:non_conform_score}
    \hat e(Y) = (\hat \varepsilon - \bar{\varepsilon})^T \widehat \Sigma_{\rho}^{-1} (\hat \varepsilon - \bar{\varepsilon}) \in \mathbb R.
\end{equation}
We then compute the non-conformity scores on the training set $(X_t,Y_t)$ for $t=1,\ldots,T$ to obtain the set 
\[\mathcal E_T = \{\hat e(Y_t)\}_{t=1}^{T}.\]
Note that non-conformity scores in $\mathcal E_T$ can be sequentially dependent due to the inherent dependency among the original data $(X_t, Y_t)$. We take this into account by using \SPCI{} \citep{xu2023sequential}, a sequential conformal inference method for univariate time series. Specifically, rather than directly taking the \textit{empirical} quantile over $\mathcal E_T$, we fit a \textit{quantile regression estimator} $\hatQt$ on $\mathcal E_T$, where $\hatQt(\alpha)$ aims to predict the $\alpha$-quantile of the unseen non-conformity score. There is no specific restriction on the quantile regression method used here. For example, \SPCI \citep{xu2023sequential} uses the quantile random forest.

We first define the set difference of two sets $A$ and $B$ as $A\setminus B=\{x:x\in A \text{ and } x\notin B\}.$ Thus, the prediction set $\Ctalpha \subset \R^p$ for a given confidence level $\alpha$ takes the form
\begin{align}
    \Ctalpha &= \{Y: \hatQt(\betahat) \leq \hat{e}(Y) \leq \hatQt(1-\alpha+\betahat)\} \label{eq:UQ_set}\\
             &=\hat{f}(X_t)+\mathcal{B}(\sqrt{\hatQt(1-\alpha+\betahat)},\bar{\varepsilon}, \widehat \Sigma_{\rho})\setminus \mathcal{B}(\sqrt{\hatQt(\betahat)}, \bar{\varepsilon}, \widehat \Sigma_{\rho}) \nonumber \\
     \betahat &= \underset{\beta \in [0,\alpha]}{\arg\min}\ V(\widehat \Sigma_{\rho}, \hatQt(1-\alpha+\beta))-V(\widehat \Sigma_{\rho}, \hatQt(\beta)) \label{eq:betahat}.
\end{align}    
In \eqref{eq:UQ_set}, the prediction region contains all $Y$ such that their non-conformity scores $\hat{e}(Y)$ are within the respective quantiles of the ellipsoid, which centers at the prediction $\hat{f}(X_t)$ as shown in the second line of Eq. \ref{eq:UQ_set}.
In \eqref{eq:betahat}, $V(\widehat \Sigma_{\rho}, r)$ denotes the volume of the ellipsoid with radius $r$, and we find $\betahat$ empirically as the tightest significance level at which the volume of the prediction region is smallest. Note that optimizing $\beta$ further allows us to consider asymmetry in the distribution of non-conformity scores. 
When the optimal $\hat{\beta}$ is zero, it reduces to an ellipsoid as follows, which is also shown in Figure \ref{fig_spci_tllustrate}(c).
\begin{equation*}
\begin{aligned}
\Ctalpha = \{Y:\hat{e}(Y) \leq \hatQt(1-\alpha)\} =\hat{f}(X_t)+\mathcal{B}\left(\sqrt{\hatQt(1-\alpha)},\bar{\varepsilon}, \widehat \Sigma_{\rho}\right) 
\end{aligned}
\end{equation*}

\begin{figure}[!t]
  \begin{minipage}{0.32\textwidth}
      \includegraphics[width=\linewidth]{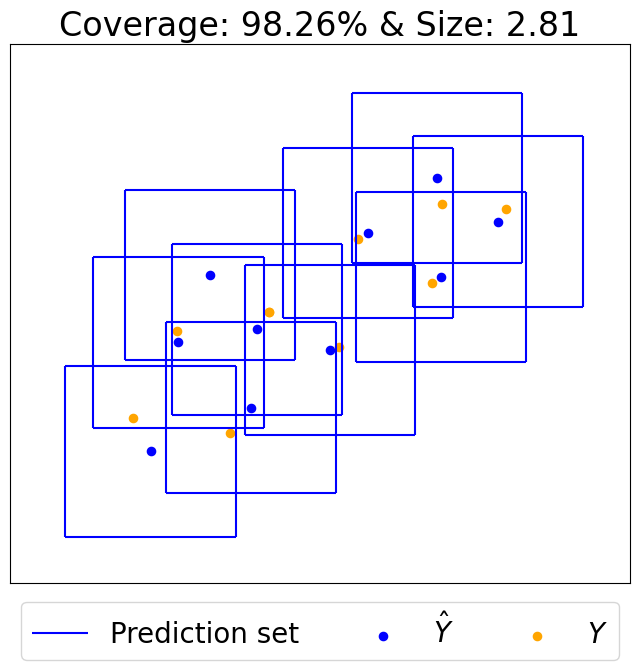}
      \subcaption{}
  \end{minipage}
  \begin{minipage}{0.32\textwidth}
      \includegraphics[width=\linewidth]{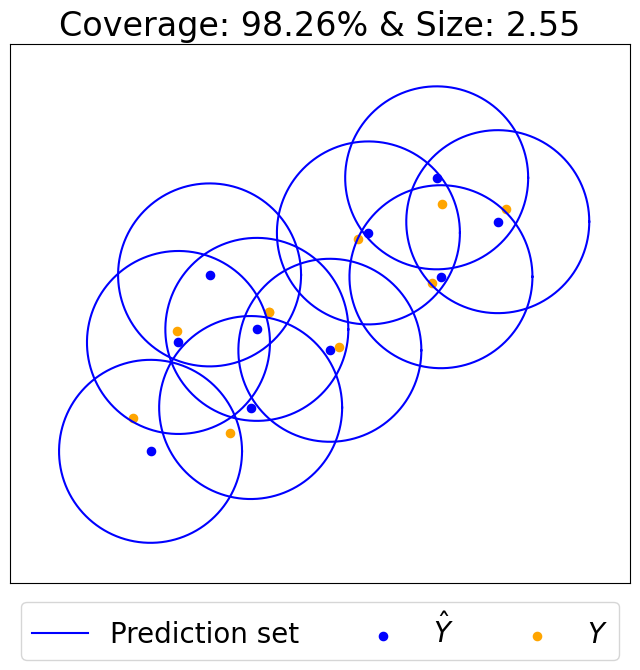}
      \subcaption{}
  \end{minipage}
  \begin{minipage}{0.32\textwidth}
      \includegraphics[width=\linewidth]{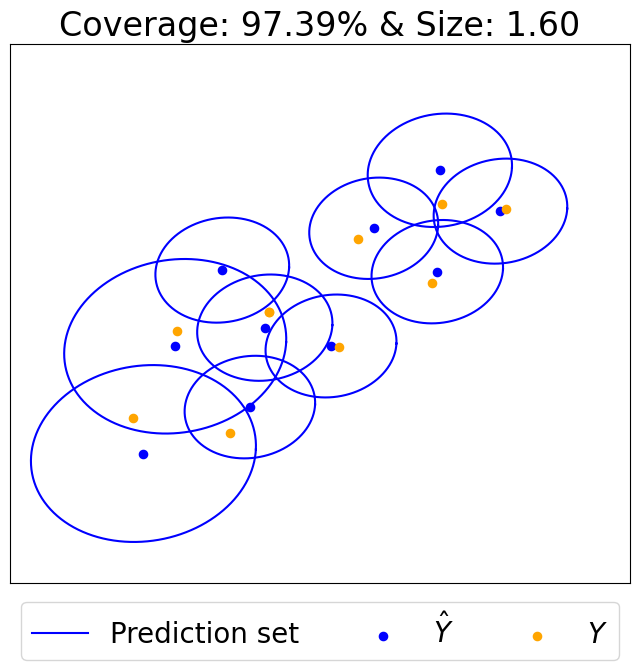}
      \subcaption{}
  \end{minipage}
  \caption{Comparison of multivariate CP method on real two-dimensional wind data (see Section \ref{sec:real_data}). {Left (a):} Empirical copula \citep{messoudi2021copula} which constructs coordinate-wise prediction intervals. {Middle (b):} Spherical confidence set introduced in \citep{sun2024copula}. {Right (c):} our proposed ellipsoidal confidence set via \MultiSPCI{}. While all methods yield coverage at least above the target 95\% on test data, our method yields the smallest average size.}
  \label{fig_spci_tllustrate}
\end{figure}

We propose \MultiSPCI{} in Algorithm \ref{algo_SPCI_multi_dim} as a multi-dimensional generalization of the original \SPCI{} \citep{xu2023sequential} method. The main benefits lie in the extension to quantify uncertainty in multi-dimensional prediction. 
The method we propose is simple and uses an ellipsoid uncertainty set. However, we will show later that this method can achieve conditional coverage. Besides, even though our method only uses the information of the first two moments, it outperforms the copula method, which is the state of the art.
Figure \ref{fig_spci_tllustrate} illustrates the benefit of \MultiSPCI{} over existing methods in yielding smaller uncertainty sets for multi-dimensional UQ time series.

\begin{algorithm}[!t]
\cprotect\caption{Multi-dimensional \SPCI{} (\MultiSPCI{})}
\label{algo_SPCI_multi_dim}
\begin{algorithmic}[1]
\REQUIRE{Training data $\{(X_t, Y_t)\}_{t=1}^T$, 
prediction algorithm $\mathcal{A}$,
significance level $\alpha$, quantile regression algorithm $\mathcal{Q}$, positive threshold $\rho>0$.
}
\ENSURE{Prediction intervals $\Ctalpha, t>T$}
\STATE Obtain $\hat{f}$ and residuals $\{\hat{ \varepsilon}_t\}_{t=1}^{T} \subset \R^p$ (computed on the holdout set) with $\mathcal{A}$ and $\{(X_t, Y_t)\}_{t=1}^T$
\STATE Compute non-conformity scores $\mathcal E_T$ from $\{\hat{ \varepsilon}_t\}_{t=1}^{T}$ and $\widehat \Sigma_{\rho}$ using \eqref{eq:non_conform_score}
\FOR {$t>T$}
\STATE Use quantile regression to obtain $\hatQt\leftarrow\mathcal{Q}(\mathcal E_T)$
\STATE Obtain uncertainty set $\Ctalpha$ as in \eqref{eq:UQ_set}
\STATE Obtain new residual $\hat \varepsilon_t$
\STATE Update residual set $\{\hat{ \varepsilon}_t\}_{t=1}^{T}$ by adding $\hat \varepsilon_t$ and removing the oldest one and update $\mathcal E_T$
\ENDFOR
\end{algorithmic}
\end{algorithm}

\subsection{Comparison with copula}\label{sec:copula}

We briefly introduce copula and explain how copula has been utilized in multivariate conformal prediction. We then highlight the key differences of the copula-based CP method with our \MultiSPCI{}.

Let $\boldsymbol{X}=(X_1,\ldots, X_p)$ 
be a generic $p$-dimensional continuous random vector with the joint CDF $F$ and marginal CDFs $F_j$ of $X_j$ for $j=1,\ldots,p$. We remark that in this subsection, for notation convenience, the subscript $_j$ in $X_j$ denotes the $j$-th component of $\boldsymbol{X}$ rather than the $j$-th feature vector of the original time series (i.e., $X_j$ in the sequence $\{(X_t, Y_t)\}$). Define $U_j:=F_j(X_j)$, where for $u\in [0,1]$, $\mathbb{P}(U_j\leq u)=\mathbb{P}(X_j\leq F^{-1}_j(u))=F_j(F^{-1}_j(u))=u$. Hence, $U_j\sim \rm{Unif}[0,1]$ is a uniform random variable on $[0,1]$. Now, the joint CDF of $(U_1,\ldots,U_p)$ is the \textit{copula} $C$ of $(X_1,\ldots,X_p)$:
\begin{align}
    C(u_1,\ldots,u_p)
    &= \mathbb{P}(U_1\leq u_1,\ldots,U_p \leq u_p) \label{eq:copula_def}\\
    &= \mathbb{P}(X_1 \leq F^{-1}_1(u_1),\ldots,X_p \leq F^{-1}_p(u_p))\nonumber \\
    &= F(F^{-1}_1(u_1),\ldots,F^{-1}_p(u_p)). \nonumber
\end{align}
Hence, the copula $C$ links $p$ marginal CDFs $\{F_j\}$ to the joint CDF $F$. 
For instance, consider bivariate Gaussian copula as an example, where we can explicitly write down the copula $C$. Let $(X_1,X_2)\sim \mathcal{N}(\boldsymbol{0},\Sigma)$ with $\Sigma_{11}=\Sigma_{22}=1$ and $\Sigma_{12}=\Sigma_{21}=\kappa$ for $\kappa\in [-1,1]$. Then, 
\begin{align}
    C(u_1,u_2)
    &= \mathbb{P}(U_1\leq u_1,U_2 \leq u_2) \nonumber\\
    &= \mathbb{P}(X_1 \leq \Phi^{-1}(u_1),X_2 \leq \Phi^{-1}(u_2))\nonumber \\
    &= \Phi_2(\Phi^{-1}(u_1),\Phi^{-1}(u_2);\kappa), \label{eq:copula_gaussian}
\end{align}
where $\Phi$ is the CDF of $\mathcal{N}(0,1)$ and $\Phi_2(\cdot;\kappa)$ is the joint CDF of $\mathcal{N}(\boldsymbol{0},\Sigma)$. Note that the bivariate Gaussian copula is parametric, assuming the marginal and joint distributions follow Gaussian distributions.

In conformal prediction, copula has been used to calibrate the coordinate-wise quantile of prediction residuals. Let $|\hat{\varepsilon}_{tj}|$ be the $j$-th coordinate of the $t$-th prediction residual in absolute value, and let $F_{tj}$ be its marginal distribution. Then, past works \citep{messoudi2021copula} fit a copula $C_t$ to the $p$-dimensional random vector $(|\hat{\varepsilon}_{t1}|,\ldots,|\hat{\varepsilon}_{tp}|)$. Specifically, they find $(u_{t1}, \ldots, u_{tp})\in [0,1]^p$ so that
\begin{align*}
    & \mathbb{P}(|\hat{\varepsilon}_{t1}|\leq F_{t1}^{-1}(u_{t1}),\ldots,|\hat{\varepsilon}_{tp}|\leq F_{tp}^{-1}(u_{tp}))
    = C_t(u_{t1}, \ldots, u_{tp}) = 1-\alpha,
\end{align*}
where $\alpha$ is a pre-specified significance level (e.g., $\alpha=0.05$). In practice, $F_{tj}$ is unknown so it is replaced by $\hat{F}_{tj}$, the empirical distribution defined using past residuals, and the values $(u_{t1}, \ldots, u_{tp})$ are found under special assumptions (e.g., $u_{t1}=\ldots=u_{tp}$ \citep{messoudi2021copula}) or searched via stochastic gradient descent \citep{sun2024copula}. 

We remark two main differences between copula conformal prediction and our proposed \MultiSPCI{}. First, the use of copula CP requires searching for multi-dimensional vectors $\boldsymbol{u}_t=(u_{t1}, \ldots, u_{tp})$ at each $t$, whose efficiency and accuracy also highly depends on the choice of copula $C_t$. How to design copula and search for the best $\boldsymbol{u}_t$ remains unclear. In contrast, our \MultiSPCI{} requires much less design effort, as it only uses an estimation of the covariance matrix of residuals $\{\hat{\varepsilon}_t\}$. Second, note that copula CP returns \textit{hyper-rectangular} prediction sets, as the method constructs one prediction interval at each $p$ coordinates. Such hyper-rectangular sets can be too large compared to ellipsoidal sets, as we experimentally find ours are significantly smaller without affecting test coverage (see Section \ref{sec:real_data}).

\subsection{Benefits of the proposed approach}\label{sec:compare}

We further discuss the benefits of \MultiSPCI{} against other approaches.

\vspace{0.05in}
\noindent \textit{Against coordinate-wise use of \SPCI{} \citep{xu2023sequential}:} Rather than building ellipsoidal uncertainty sets, a naive but perhaps more intuitive approach is to apply \SPCI{} $p$ times, once per dimension of $Y\in \R^p$. The resulting uncertainty sets are hyper-rectangles, which can be unnecessarily large in many cases. In addition, the significance values for \SPCI{} at different dimensions need to be adjusted appropriately to achieve valid coverage of $Y$. Computationally, such use of \SPCI{} is also more expensive than \MultiSPCI{} because $(p-1)T_1$ additional quantile regression models are fitted ($T_1$ is the length of the test set). 

\vspace{0.05in}
\noindent \textit{Against copula-based CP methods \citep{messoudi2021copula,sun2024copula}:} Besides the limitation above of returning hyper-rectangular uncertainty sets, these copula-based methods fail to account for the sequential dependency of non-conformity scores when taking the empirical quantile over scores. In contrast, the proposed \MultiSPCI{} explicitly takes dependency into account by adaptively re-estimating the quantile of non-conformity scores.

\vspace{0.05in}
\noindent \textit{Against probabilistic forecasting methods \citep{salinas2020deepar,lim2021temporal}:} There are two main benefits of \MultiSPCI{}. First, our proposed method is compatible with any user-specified prediction model $\hat{f}$ of $Y$. In contrast, such probabilistic forecasting methods often require specifically designed deep neural networks to predict the quantiles of $Y$ directly. Second, we can provide coverage guarantees for the proposed method, whereas those methods often lack sound justifications.

\subsection{Improvements using local ellipsoids}
In practice, constructing the scalar non-conformity scores in \eqref{eq:non_conform_score} based on a global covariance matrix in \eqref{eq:global_cov} fails to capture local variation in data. We can improve our \MultiSPCI{} through using local ellipsoids proposed in \citep{messoudi2022ellipsoidal}.

Specifically, given a test data feature $X_t$ for $t>T$, we first consider its $k$ nearest neighbors among previous $T$ samples $\{X_{t-1},X_{t-2},\ldots,X_{t-T}\}$. Let the index set of neighbors be $N_t$ with $|N_t|=k$. Then, we denote $\widehat{\rm{Cov}}_t$ as the sample covariance estimator using $\{\hat{\varepsilon}_t\}_{t\in N_t}$ (see Eq. \eqref{eq:global_cov} for the definition using $\{\hat{\varepsilon}_t\}_{t=1}^T$). As a result, given a parameter $\lambda \in [0,1]$, the local covariance estimator at time $t$ is written as the weighted average
\begin{equation}\label{eq:local_cov}
    \widehat{\Sigma}_t=\lambda \widehat{\rm{Cov}}_t + (1-\lambda) \widehat{\Sigma},
\end{equation}
where $\widehat{\Sigma}$ is the global empirical covariance matrix in \eqref{eq:global_cov}. We recommend setting $k=0.1T$ and $\lambda=0.95$ to capture local variations effectively. Lastly, as $\widehat{\Sigma}_t$ in \eqref{eq:local_cov} may not be invertible, we use its low-rank approximation in \eqref{eq:cov_est_rank_k} and the corresponding pseudo-inverse in \eqref{eq:pseudo_inv}, which would be used to compute the non-conformity score \eqref{eq:non_conform_score} at time $t$. We empirically find that compared to using \eqref{eq:global_cov}, the use of \eqref{eq:local_cov} can lead to up to 25\% reduction in the average size of prediction sets.

\section{Theoretical Analysis}
\label{TA}

In this section, we will present theoretical results for bounding the conditional coverage of our method. The result is based on the case where the sample covariance matrix is invertible. We first recall and define the notations and then give out the assumptions required, which are general and identifiable. After that, we will present our coverage guarantees when using the empirical quantile function as the quantile regression predictor. The norm $\|\cdot\|$ used in the paper is the spectral norm ($2$-norm). The proof details will be in Appendix \ref{AppP}. The main idea of the proof consists of two parts. The first part is the convergence of the empirical CDF to the true CDF of the residual. The second part is to control the estimation error of the sample covariance matrix.

We assume that $Y_t\in\R^p$ is generated from a true model with unknown additive noise:
\begin{equation}\label{eq:dgp}
Y_t=f(X_t)+\varepsilon_t, \quad t = 1, 2, \ldots,
\end{equation}
where $f$ is an unknown function and $\varepsilon_t$ represents the process noise, whose marginal distribution is not necessarily Gaussian, the process noise may have temporal dependence. For the simplicity of notation, we assume $\bar{\varepsilon}=0$ here so that the non-conformity score simplifies to $\hat{\varepsilon}_t^T\widehat{\Sigma}^{-1}\hat{\varepsilon}_t$. Besides, without loss of generality, we can assume $\mathbb{E}[\varepsilon]=0$. Otherwise, we can subtract the mean from the noises and add to the function $f$.

We define $\hat{\varepsilon}_t=y_t-\hat{f}(x_t)$ as the {\it vector prediction residual}, and the {\it scalar non-conformity score}
\[\hat{e}_t=\hat{\varepsilon}_t^T\widehat \Sigma^{-1}\hat{\varepsilon}_t \in \mathbb R.\] 
Moreover, 
\[e_t=\varepsilon_t^T\Sigma^{-1}\varepsilon_t, 
\quad \Delta_t=\hat{\varepsilon}_t-\varepsilon_t.\] Here $\varepsilon_t\in\R^{p}$ is the true noise in model \eqref{eq:dgp} and $\Sigma\in\R^{p\times p}$ is the true covariance matrix of $\varepsilon$. Besides, we define the empirical CDF
\begin{equation}
\begin{aligned}
\widehat{F}_{T+1}(x)=\frac{1}{T}\sum_{t=1}^{T}1\{\hat{e}_t\le x\}, \quad 
\widetilde{F}_{T+1}(x)=\frac{1}{T}\sum_{t=1}^{T}1\{e_t\le x\}.
\end{aligned}
\end{equation}
We also use 
\[F_e(x)=\p\{e\le x\},\] to represent the CDF of the nonconformity score. Since we consider the case where the marginal distributions of $e_t$ are identical, their CDF is the same, and we can define it as $F_e(x)$. In our method, we use the empirical distribution of non-conformity score $\hat{e}$ to approximate the distribution of $e$. A new observation $Y_{T+1}$ being covered by the conformal interval with given coverage is equivalent to $\hat{e}_{T+1}$ falling in a given quantile in empirical distribution $\widehat{F}_{T+1}$.

From the property of CDF, we know that $F_e(e_{T+1})\sim\mathrm{Unif}[0,1]$. If we can show that $\widehat{F}_{T+1}$ approximates $F_e$ well, then it will follow that $\widehat{F}_{T+1}$ approximately covers a region of $1-\alpha$ probability. Comparing $\hat{e}$ and $e$, we see that
\begin{equation}
\hat{e}_t-e_t=\hat{\varepsilon}_t^T\widehat \Sigma^{-1}\hat{\varepsilon}_t-\varepsilon_t^T\Sigma^{-1}\varepsilon_t,
\end{equation}
where $\widehat \Sigma$ is estimated from $\{\hat{\varepsilon}_t\}_{t=1}^{T}$. We would need $\hat{\varepsilon}_t$ to be close to $\varepsilon_t$. Otherwise, this approximation would not hold. 
\begin{assumption}[i.i.d. and Lipschitz]
\label{a1}
Assume $\{\varepsilon_t\}_{t=1}^{T+1}$ are independent and identically distributed (i.i.d.). Meanwhile, $F_e(x)$ (the CDF of the true non-conformity score) is assumed to be Lipschitz continuous
with constant $L_{T+1}>0$. 
\end{assumption}
\begin{remark}
We first assume that the error process $\{\varepsilon_t\}_{t=1}^{T+1}$ is i.i.d. In fact, this assumption is not necessary, and we will extend this assumption to cases beyond exchangeability. The result for stationary and strong mixing sequences will be presented in Corollary \ref{c1}. 
\end{remark}

\begin{assumption}[Estimation quality]
\label{a2}
There exists a sequence $\{\delta_T\}_{T\ge 1}$ such that
\begin{equation}
\begin{aligned}
\frac{1}{T}\sum_{t=1}^{T}\|\Delta_t\|^2 \le \delta_{T}^{2}, ~~\|\Delta_{T+1}\| \le \delta_T.
\end{aligned}
\end{equation}
\end{assumption}
\begin{remark}
The assumption requires that the square sum of the prediction error be bounded by $\delta_T^2$. For many estimators, there exists a sequence $\{\delta_T\}_{T\ge 1}$ that goes to zero. For example, $\delta_T=o_p(T^{-1/4})$ for general neural networks sieve estimators when $f$ is sufficiently smooth \citep{chen1999improved}. When $f$ is a sparse high-dimensional linear model, $\delta_T=o_p(T^{-1/2})$ for the
Lasso estimator and Dantzig selector \citep{bickel2009simultaneous}. 
\end{remark}

\begin{assumption}[Covariance eigenvalue]
\label{a3}
There exists a $\lambda>0$ s.t. $\|\Sigma\|\ge \lambda$ and $\|\widehat \Sigma\|\ge \lambda$.
\end{assumption}
\begin{remark}
It is a common assumption to require both the covariance matrix and the estimated covariance matrix to be strictly positive definite. The condition holds for the covariance matrix as long as there is no linear dependency between variables, which is true for the errors $\varepsilon_t$. Besides, the assumption is also satisfied by the sample covariance matrix because our algorithm uses the pseudo-inverse, which ensures positive eigenvalues.
\end{remark}

\begin{assumption}[Tail behavior]
\label{a4}
There exist some constants $q>4$, $K_1, K_2>0$ and $L\ge 1$, such that $\max_{t\le T} \|\varepsilon_t\|\le \sqrt{K_1p}$ almost surely and $\mathbb{E}|\langle\varepsilon,x\rangle|^q\le L^q$ for $x\in S^{p-1}$. Besides, there exists a constant $K_2$ such that $\mathrm{Var}[\|\varepsilon\|^2]\le K_2p$. 
\end{assumption}
\begin{remark}
The assumption is required in Theorem 1.2 \citep{vershynin2012close} so that the sample covariance matrix converges to the true covariance matrix in the operator norm. Other assumptions in literature ensure a $O(T^{-{1/2}})$ convergence. For example, \citep{koltchinskii2017concentration} requires random variables $\varepsilon$ to be weakly square-integrable, sub-Gaussian, and pre-Gaussian. Our method can use covariance estimators other than the classic sample covariance matrix. The sample covariance matrix is a natural choice, but it is a poor estimator when the dimension is very high unless there are some nice tail behaviors \citep{lugosi2019near}. There are a lot of results in the literature focusing on covariance estimation under different conditions. The Assumption \ref{a4} can be easily switched to other requirements if we change the estimator. \citep{cai2016estimating} offers an overview of covariance estimators with their optimal rates. The choice and analysis of the covariance matrix is not the main focus of our paper, so we use the sample covariance matrix here for simplicity.
\end{remark}

With the i.i.d. assumption, we can show that the empirical distribution of $e_t$ approximates the true CDF well in the following sense.
\begin{lemma}[Convergence of empirical CDF of $\{\varepsilon_t\}_{t=1}^{T}$ under i.i.d.]
\label{l1}
Under Assumption \ref{a1}, for any training size T, there is an event $A_T$ which occurs
with probability at least $1-\sqrt{\frac{\log(16T)}{T}}$, such that conditioning on $A_T$,
\begin{equation}
\sup_x |\widetilde{F}_{T+1}(x)-F_e(x)|\le \sqrt{\frac{\log(16T)}{T}}.
\end{equation}
\end{lemma}
\begin{remark}   
 The i.i.d. assumption is not a must, and we can easily extend it to the case where $\{e_t\}_{t=1}^{T}$ is stationary and strong mixing. We will show a similar result in Corollary \ref{sta_approx}.
\end{remark}

With the assumptions, we can also show that the empirical distribution of $\hat{e}$ approximates the empirical distribution of $e$ well in the following sense:
\begin{lemma}[Distance between the empirical CDF of $\{\varepsilon_t\}_{t=1}^{T}$ and $\{\hat{\varepsilon}_t\}_{t=1}^{T}$ under i.i.d.]
\label{l2}
Under Assumption \ref{a1}, \ref{a2}, \ref{a3} and \ref{a4}, with a high probability $1-\delta$,
\begin{equation}
\sup_x |\widehat{F}_{T+1}(x)-\widetilde{F}_{T+1}(x)| \le(L_{T+1}+1) C_S+ 2 \sup _x|\widetilde{F}_{T+1}(x)-F_e(x)|,  
\end{equation}
where 
\begin{equation}
\begin{aligned}
C_S &= \left(\frac{\delta_T^2}{\lambda}+ \frac{K_3}{\lambda^2}\left[C\left(\frac{1}{\delta}\right)^{20/9q}\log\left(\frac{1}{\delta}\right)(\log\log p)^2 \frac{p^{3/2-2/q}}{T^{1/2-2/q}} + 22K_3 \max\left\{\frac{p^3}{T\delta},p^{3/2}\delta_T\right\}\right]\right)^{1/2}\\
&=\tilde{O}\left(\max\left\{\frac{p^{3/4-1/q}}{T^{1/4-1/q}},p^{3/4}\delta_T^{1/2}\right\}\right),\label{cs}
\end{aligned}
\end{equation}
and $C$ is a constant that depends only on $K_1, q, L$ and $K_3=K_1+\sqrt{3K_2}$ is a constant.
\end{lemma}

 Our main theorem is the following Theorem \ref{t1}, which establishes the asymptotic conditional coverage as a result of Lemma \ref{l1} and \ref{l2}.
\begin{theorem}[Conditional guarantee under i.i.d. assumption]
\label{t1}
Under Assumption \ref{a1}, \ref{a2}, \ref{a3} and \ref{a4}, with a high probability $1-\delta$, for any training size $T$ and $\alpha\in(0,1)$, we have
\begin{equation}
\begin{aligned}
&|\p(Y_{T+1} \in \widehat{C}_{T+1}^\alpha \mid X_{T+1}=x_{T+1})-(1-\alpha)|\\
&\qquad \leq 12 \sqrt{\frac{\log(16T)}{T}}+4(L_{T+1}+1)(C_S+\delta_T),
\end{aligned}
\end{equation}
where $C_S$ is defined in \eqref{cs}.
\end{theorem}
\begin{remark}
The bound is controlled by the sample size $T$ and the coefficient $\delta_T$. This vanishes when $T\rightarrow\infty$ and $\delta_T\rightarrow 0$, which means that when the sample size is large enough, and the estimator $\hat{f}$ is accurate enough, the conditional coverage will converge to $1-\alpha$.
\end{remark}

\begin{corollary}[Guarantee with true covariance matrix, and under i.i.d.]
If the true covariance matrix $\Sigma$ is known, we can use $\widehat \Sigma = \Sigma$.  Under Assumption \ref{a1}, \ref{a2}, \ref{a3} and \ref{a4}, for any training size $T$ and $\alpha\in(0,1)$, we have
\begin{equation}
\begin{aligned}
&|\p(Y_{T+1} \in \widehat{C}_{T+1}^\alpha \mid X_{T+1}=x_{T+1})-(1-\alpha)| \\
&\qquad \leq 12 \sqrt{\frac{\log(16T)}{T}}+4(L_{T+1}+1)\left(\frac{\delta_T}{\sqrt{\lambda}}+\delta_T\right).
\end{aligned}
\end{equation}    
\end{corollary}
\begin{remark}
When the true covariance matrix is known, we have a tighter and simpler bound than Theorem \ref{t1}. Although the true covariance is usually unknown in reality, the same bound can be reached if the sample covariance matrix is estimated from another independent training set.
\end{remark}

Then, we would present a corollary that extends our guarantee to the case where $\{\varepsilon_t\}_{t=1}^{T}$ is a stationary and strong mixing sequence. 
\begin{definition}
A sequence of random variables $\{X_n\}$ is said to be {\it strictly stationary} if for every $k \geq 1$, any integers $n_1, \ldots, n_k$, and any integer $h$, the joint distribution of the random variables $(X_{n_1}, \ldots, X_{n_k})$ is the same as the joint distribution of $(X_{n_1+h}, \ldots, X_{n_k+h})$.    
\end{definition}

\begin{definition}
A sequence of random variables $\{X_n\}$ is said to be {\it strongly mixing} (or $\alpha$-mixing) if the mixing coefficients $\alpha_k$ defined by
\[
\alpha_k = \sup_{n \in \mathbb{N}} \sup_{A \in \mathcal{F}_1^n, B \in \mathcal{F}_{n+k}^\infty} |P(A \cap B) - P(A)P(B)|
\]
tend to zero as $k \to \infty$, where $\mathcal{F}_a^b$ denotes the $\sigma$-algebra generated by $\{X_a, \ldots, X_b\}$.
    
\end{definition}

Using a similar technique, we can prove the following result for the case where $\{\varepsilon_t\}_{t=1}^{T}$ is a stationary and strong mixing sequence. Here, we assume that the true covariance matrix $\Sigma$ is known for simplicity, but we will discuss in Remark \ref{r1} how to extend the result to the case where true covariance is unknown.

\begin{corollary}[Guarantee with true covariance matrix, under stationarity and strong mixing]
\label{c1}
Assume $\{\varepsilon_t\}_{t=1}^{T}$ is a stationary and strong mixing sequence with mixing coefficient $0<\sum_{k>0}\alpha_k <M$. Under Assumption \ref{a2}, \ref{a3} and \ref{a4}, for any training size $T$ and $\alpha\in(0,1)$, we have
\begin{equation}
\begin{aligned}
&|\p(Y_{T+1} \in \widehat{C}_{T+1}^\alpha \mid X_{T+1}=x_{T+1})-(1-\alpha)| \\
&\qquad \leq 12 \frac{(\frac{M}{2})^{1/3}(\log T)^{2/3}}{T^{1/3}}+4(L_{T+1}+1)\left(\frac{\delta_T}{\sqrt{\lambda}}+\delta_T\right).
\end{aligned}
\end{equation}
\end{corollary}
\begin{remark}
\label{r1}
The first term in the convergence rate is of order $\tilde O(T^{-1/3})$, which is slighter bigger than the order $\tilde O(T^{-1/2})$ in Theorem \ref{t1} under i.i.d. case. The second term is the same. The essence of generalizing the outcome to scenarios where the true covariance matrix $\Sigma$ remains unknown lies in the convergence properties of the sample covariance matrix. There are works directed towards these convergence properties within the context of stationary time series. For instance, \cite{chen2013covariance} presented an asymptotic convergence result for a threshold sample covariance estimator. Utilizing methodologies akin to those employed in Theorem \ref{t1}, a similar result can be substantiated, and there is also the possibility to apply a variety of estimators. However, the focus of this work is not on the convergence of the covariance matrix estimator; hence, further exploration in this direction is omitted. Moreover, as previously indicated, independent training data can be leveraged to estimate the covariance matrix and achieve the same bound in Corollary \ref{c1}.
\end{remark}

\begin{table*}[!t]
    \centering
    \caption{Simulation results by both methods. Target coverage is 90\%. Standard deviation is computed over ten independent trials in which training and test data are regenerated.}
    \label{tab:simulation}
    \begin{minipage}{\textwidth}
        \subcaption{Independent ${\rm AR}(w)$}
        \resizebox{\textwidth}{!}{
        \begin{tabular}{P{1.45cm}|P{1.7cm}P{1.7cm}|P{1.7cm}P{1.7cm}|P{1.7cm}P{1.7cm}|P{1.7cm}P{1.7cm}|P{1.7cm}P{1.7cm}|P{1.7cm}P{1.7cm}}
        \toprule
         $p$ & \multicolumn{2}{c}{2} & \multicolumn{2}{c}{4} & \multicolumn{2}{c}{8} & \multicolumn{2}{c}{10} & \multicolumn{2}{c}{16} & \multicolumn{2}{c}{20} \\
        Method & \texttt{MultiDim} \texttt{SPCI} & \SPCI{} & \texttt{MultiDim} \texttt{SPCI} & \SPCI{} & \texttt{MultiDim} \texttt{SPCI} & \SPCI{} & \texttt{MultiDim} \texttt{SPCI} & \SPCI{} & \texttt{MultiDim} \texttt{SPCI} & \SPCI{} & \texttt{MultiDim} \texttt{SPCI} & \SPCI{} \\
        \midrule
        Coverage & 90.0\% (0.26) & 90.0\% (0.29) & 90.0\% (0.25) & 89.9\% (0.14) & 90.0\% (0.31) & 89.9\% (0.30) & 89.8\% (0.25) & 89.8\% (0.27) & 89.9\% (0.24) & 89.9\% (0.23) & 90.0\% (0.26) & 89.8\% (0.30) \\
        Size & 1.45e+1 (9.34e-2) & 1.52e+1 (8.73e-2) & 3.00e+2 (2.62e+0) & 3.94e+2 (3.38e+0) & 1.30e+5 (1.43e+3) & 3.68e+5 (6.44e+3) & 2.65e+6 (4.79e+4) & 1.22e+7 (1.61e+5) & 2.23e+10 (5.61e+8) & 5.84e+11 (1.39e+10) & 9.15e+12 (2.97e+11) & 8.67e+14 (2.90e+13) \\
        \bottomrule
        \end{tabular}}
    \end{minipage}
    
    \vspace{0.05in}
    \begin{minipage}{\textwidth}
        \subcaption{${\rm VAR}(w)$}
        \resizebox{\textwidth}{!}{
        \begin{tabular}{P{1.45cm}|P{1.7cm}P{1.7cm}|P{1.7cm}P{1.7cm}|P{1.7cm}P{1.7cm}|P{1.7cm}P{1.7cm}|P{1.7cm}P{1.7cm}|P{1.7cm}P{1.7cm}}
        \toprule
         $p$ & \multicolumn{2}{c}{2} & \multicolumn{2}{c}{4} & \multicolumn{2}{c}{8} & \multicolumn{2}{c}{10} & \multicolumn{2}{c}{16} & \multicolumn{2}{c}{20} \\
        Method & \texttt{MultiDim} \texttt{SPCI} & \SPCI{} & \texttt{MultiDim} \texttt{SPCI} & \SPCI{} & \texttt{MultiDim} \texttt{SPCI} & \SPCI{} & \texttt{MultiDim} \texttt{SPCI} & \SPCI{} & \texttt{MultiDim} \texttt{SPCI} & \SPCI{} & \texttt{MultiDim} \texttt{SPCI} & \SPCI{} \\
        \midrule
        Coverage & 90.0\% (0.26) & 92.7\% (0.25) & 90.2\% (0.21) & 91.5\% (0.22) & 90.0\% (0.23) & 91.6\% (0.18) & 89.9\% (0.23) & 90.7\% (0.31) & 89.9\% (0.20) & 91.0\% (0.19) & 90.0\% (0.25) & 90.9\% (0.19) \\
        Size & 2.73e+0 (1.36e-2) & 6.46e+0 (5.84e-2) & 3.89e+1 (2.25e-1) & 4.94e+2 (7.49e+0) & 7.16e+4 (7.25e+2) & 9.27e+6 (1.46e+5) & 3.63e+7 (4.79e+5) & 3.24e+9 (6.09e+7) & 8.55e+12 (1.45e+11) & 1.91e+17 (5.38e+15) & 1.14e+16 (2.11e+14) & 7.41e+22 (1.68e+21) \\
        \bottomrule
        \end{tabular}}
        \end{minipage}
\end{table*}

\section{Experiments}

We test Algorithm \ref{algo_SPCI_multi_dim} on both simulated and real-time series to show that \MultiSPCI{} reaches valid coverage with smaller prediction regions. In all our experiments, the value of $\rho$ used in \eqref{eq:cov_est_rank_k} is set to be 0.001. For simplicity, we only consider the global covariance matrix in \eqref{eq:global_cov} rather than its local variant \eqref{eq:local_cov}, which would bring further improvements. Code is available at \url{https://github.com/hamrel-cxu/MultiDimSPCI}.

\begin{figure*}[!t]
\vspace{0.05in}
\begin{minipage}{\textwidth}
    \includegraphics[width=\linewidth]{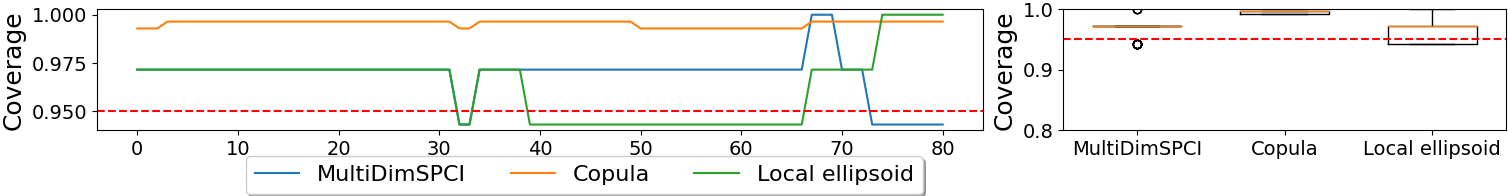}
    \includegraphics[width=\linewidth]{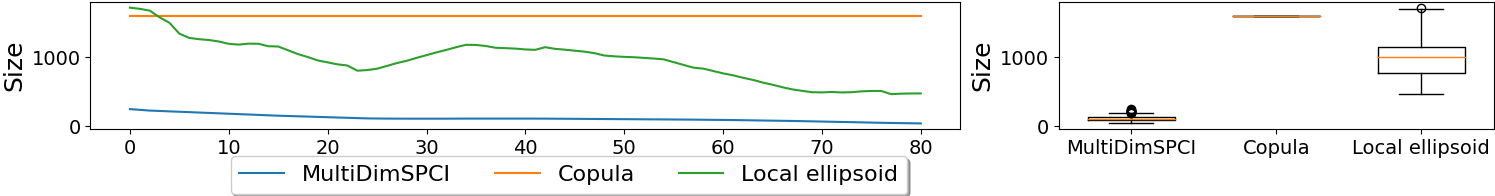}
    \subcaption{Wind data}
\end{minipage}
\begin{minipage}{\textwidth}
    \includegraphics[width=\linewidth]{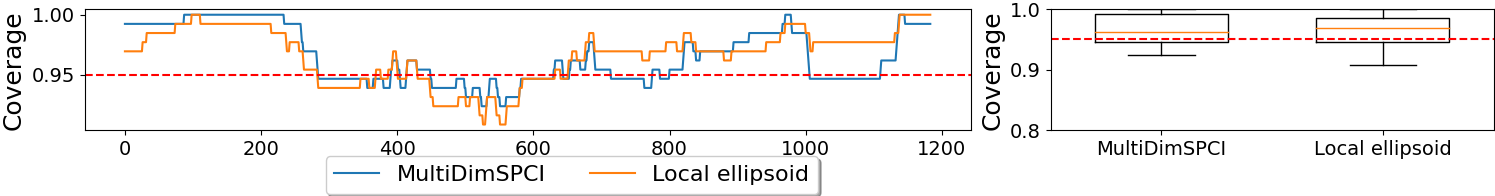}
    \includegraphics[width=\linewidth]{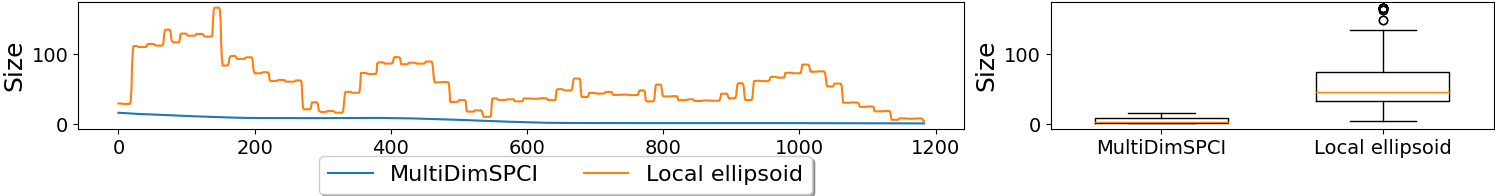}
    \subcaption{Solar data}
\end{minipage}
\begin{minipage}{\textwidth}
    \includegraphics[width=\linewidth]{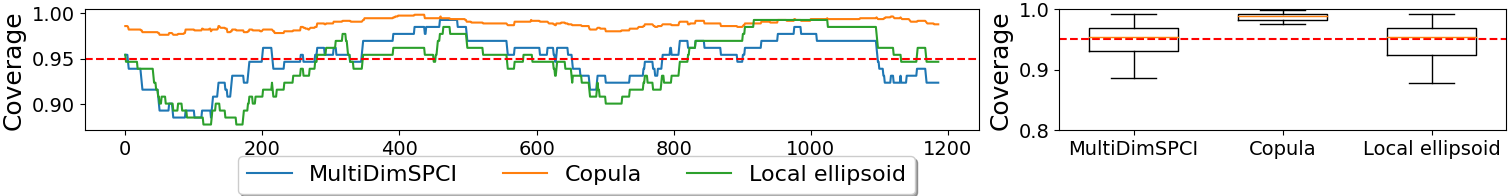}
    \includegraphics[width=\linewidth]{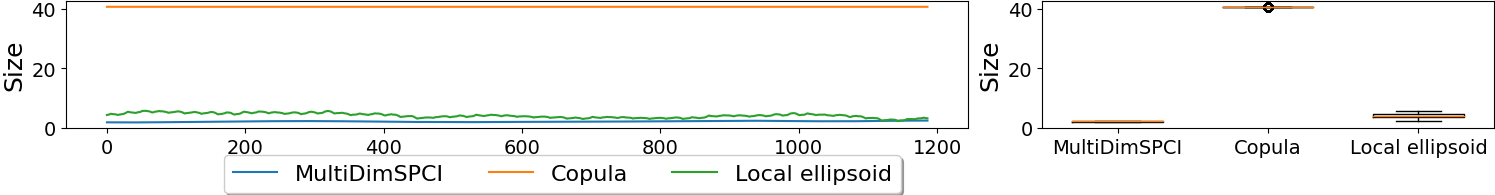}
    \subcaption{Traffic data}
\end{minipage}
\vspace{-0.075in}
\caption{Real-data comparison of rolling coverage (target coverage is 95\%) and size of prediction sets at $p=8$. In each subplot of (a)-(c), the top row plots rolling coverage over prediction time indices (red dashed line is the target coverage) and as boxplots, and the bottom row shows results for rolling sizes. We only visualize the comparison of \MultiSPCI{} with selected CP baselines, which have comparable average size of prediction regions in Table \ref{tab:real_data}.}
\label{fig:real_data_rolling}
\vspace{-0.05in}
\end{figure*}

\begin{table*}[!t]
\caption{
Real-data comparison of test coverage and average prediction set size by different methods. The target coverage is 0.95, and at each $p$, the smallest size of prediction sets is in \textbf{bold}. Our \MultiSPCI{} yields the narrowest confidence sets without sacrificing coverage for two reasons. First, it explicitly captures dependency among coordinates of $Y_t$ by forming ellipsoidal prediction sets. Second, it captures temporal dependency among non-conformity scores upon adaptive re-estimation of score quantiles.}\label{tab:real_data}
\begin{minipage}{\textwidth}
\subcaption{Wind data}
\centering
\resizebox{\textwidth}{!}{\begin{tabular}{c|c|c|c|c|c|c}
\hline
Method & $p=2$ coverage & $p=2$ size & $p=4$ coverage & $p=4$ size & $p=8$ coverage & $p=8$ size \\
\hline
{\MultiSPCI{}} & 0.97 & \textbf{1.60} & 0.96 & \textbf{7.02}& 0.96 & \textbf{72.10} \\
CopulaCPTS \citep{sun2024copula} & 0.98 &	2.55 &	0.97&	10.23&	0.97	&252.67\\
Local ellipsoid \citep{messoudi2022ellipsoidal} & 0.96 & 3.51 & 0.97 & 13.07 & 0.98 & 1.09e+3 \\
Copula \citep{messoudi2021copula} & 0.98 & 2.81 & 0.98 & 10.32 & 0.97 & 1.60e+3 \\
TFT \citep{lim2021temporal} & 0.94 & 10.61 & 0.75 & 159.39 & 0.94 & 2.91e+4 \\
DeepAR \citep{salinas2020deepar} & 0.96 & 7.07 & 0.76 & 67.97 & 0.96 & 1.79e+5 \\

\hline
\end{tabular}}
\end{minipage}
\begin{minipage}{\textwidth}
\subcaption{Solar data}
\centering
\resizebox{\textwidth}{!}{\begin{tabular}{c|c|c|c|c|c|c}
\hline
Method & $p=2$ coverage & $p=2$ size & $p=4$ coverage & $p=4$ size & $p=8$ coverage & $p=8$ size \\
\hline
{\MultiSPCI{}} & 0.96 & 1.68 & 0.96 & \textbf{2.89} & 0.97 & \textbf{4.97} \\
CopulaCPTS \citep{sun2024copula} & 0.99	&4.36	&0.99	&37.56	&0.99	&3.28e+3\\
Local ellipsoid \citep{messoudi2022ellipsoidal}& 0.97 & \textbf{1.32} & 0.97 & 3.20 & 0.97 & 43.07 \\
Copula \citep{messoudi2021copula}& 0.99 & 4.11 & 0.99 & 27.73 & 0.99 & 1.42e+3 \\
TFT \citep{lim2021temporal}& 0.99 & 13.68 & 0.99 & 71.72 & 0.93 & 1.19e+3 \\
DeepAR \citep{salinas2020deepar}& 0.97 & 10.76 & 0.98 & 157.09 & 0.74 & 31.82 \\

\hline
\end{tabular}}
\end{minipage}
\begin{minipage}{\textwidth}
\subcaption{Traffic data}
\centering
\resizebox{\textwidth}{!}{\begin{tabular}{c|c|c|c|c|c|c}
\hline
Method & $p=2$ coverage & $p=2$ size & $p=4$ coverage & $p=4$ size & $p=8$ coverage & $p=8$ size \\
\hline
{\MultiSPCI{}} & 0.96 & \textbf{1.31} & 0.96 & \textbf{1.93} & 0.96 & \textbf{2.98} \\
CopulaCPTS \citep{sun2024copula} & 0.95&	1.70	&0.94	&3.15&	0.95	&14.10\\
Local ellipsoid \citep{messoudi2022ellipsoidal}& 0.95 & 1.36 & 0.94 & 2.08 & 0.95 & 4.13 \\
Copula \citep{messoudi2021copula}& 0.95 & 1.44 & 0.95 & 3.90 & 0.94 & 40.60 \\
TFT \citep{lim2021temporal}& 0.89 & 9.07 & 0.93 & 87.92 & 0.88 & 9.69e+2 \\
DeepAR \citep{salinas2020deepar}& 0.87 & 13.53 & 0.88 & 57.20 & 0.82 & 9.89e+3 \\

\hline
\end{tabular}}
\end{minipage}
\vspace{-0.1in}
\end{table*}

\subsection{Simulation}

We simulate two types of stationary time series. The first case considers independent ${\rm AR}(w)$ sequences, and the second case considers ${\rm VAR}(w)$ sequences. We want to show that compared to \SPCI{} applied independently to each dimension (equivalent to using independent copula \citep[See Sec. 3.3.1]{messoudi2021copula}), \MultiSPCI{} yields significantly smaller prediction regions without sacrificing coverage.

\vspace{0.05in}
\noindent \textit{Data generation.} Denote $Y_t=[Y_{i1},\ldots,Y_{ip}]^T\in \R^p$ for $p \geq 2$. We generate $Y_t$  as
\begin{equation}\label{eq:simulation_dgp}
    Y_t= \sum_{l=1}^w \boldsymbol{\alpha}_l Y_{i-l} + \boldsymbol{\varepsilon}_t, ~~\boldsymbol{\varepsilon}_t \sim N(0,\Sigma).
\end{equation}
In \eqref{eq:simulation_dgp}, $\boldsymbol{\alpha}_l \in \R^{p\times p}$ contains the set of coefficients, where we further construct them so that the sequences $\{Y_t\}$ are stationary.
In the first case of independent ${\rm AR}(w)$ sequences, we have $\Sigma=I_p$. In the second case of ${\rm VAR}(w)$ sequences, we design $\Sigma=BB^T$ to be a positive definite covariance matrix, where $B_{ij} \overset{i.i.d.}{\sim} \mathrm{Unif}[-1,1]$.

\vspace{0.05in}
\noindent \textit{Setup.} In both cases of AR and VAR time series following \eqref{eq:simulation_dgp}, we let $w=5$ and vary $p\in\{2,4,8,10,16,20\}.$ The initial 80K samples $\{Y_t\}$ are training data; the remaining 20K samples are test data. Because \SPCI{} assumes independence across different univariate sequence, we let $\tilde{\alpha}=1-(1-\alpha)^{1/p}$ and apply \SPCI{} on individual sequences with the corrected $\tilde{\alpha}$. The multivariate linear regression method is used as the point predictor.

\vspace{0.05in}
\noindent \textit{Results.} Table \ref{tab:simulation} examines the empirical coverage and average size of prediction regions in $\R^p$ by both methods on the two cases of data generation. Both methods can maintain valid coverage around the target 90\% in two cases. Nevertheless, it is clear that as dimension $p$ increases, the average size of prediction regions by the proposed \MultiSPCI{} is significantly smaller (for several magnitudes) than that by \SPCI{} applied to individual sequences.
In Figure \ref{fig:non_critical_region_simul}, we further visualize the non-critical regions in both cases to demonstrate why \MultiSPCI{} provides smaller prediction regions.

\subsection{Real-data}\label{sec:real_data}

We now compare \MultiSPCI{} with existing methods designed for multivariate uncertainty quantification. The three CP baselines are CopulaCPTS \citep{sun2024copula}, Local ellipsoid \citep{messoudi2022ellipsoidal}, and Copula \citep{messoudi2021copula}. The two probabilistic forecasting baselines are temporal fusion transformers (TFT) \citep{lim2021temporal} and DeepAR \citep{salinas2020deepar}.

We consider three real multivariate time-series datasets. The first \textit{wind} dataset considers wind speed in meters per second at different wind farms \citep{zhu2021multi}, with 764 observations in total. The second \textit{solar} dataset considers solar radiation in Diffused Horizontal Irradiance (DHI) units at different solar sensors \citep{zhang2021solar}, with 8755 observations in total. The third \textit{traffic} dataset considers traffic flow collected at different traffic sensors \citep{xu2021ECAD}, with 8778 observations in total. On each dataset, we randomly select $p\in \{2,4,8\}$ locations (same for all methods) and examine the test coverage and average size on the $p$-dimensional time series. The first 85\% data are used for training, and the remaining 15\% are used for testing.

\begin{table}[!t]
    \centering
    \caption{Comparison on wind data when dimension $d=25$. The setup is identical to that in Table \ref{tab:real_data}.}
    \label{tab:high_d_tab}
    \resizebox{0.7\textwidth}{!}{
    \begin{tabular}{c|c|c|c|c}
    \hline
    & \MultiSPCI{} & CopulaCPTS & Local ellipsoid & Copula \\ 
    \hline
    Coverage & 0.95 & 0.98 & 0.98 & 0.94\\
    Size & 3.55e+7&	1.13e+14	&4.93e+16&	1.20e+13\\
    \hline
    \end{tabular}}
\end{table}
\begin{figure}[!b]
    \centering
    \begin{minipage}{0.32\textwidth}
        \includegraphics[width=\linewidth]{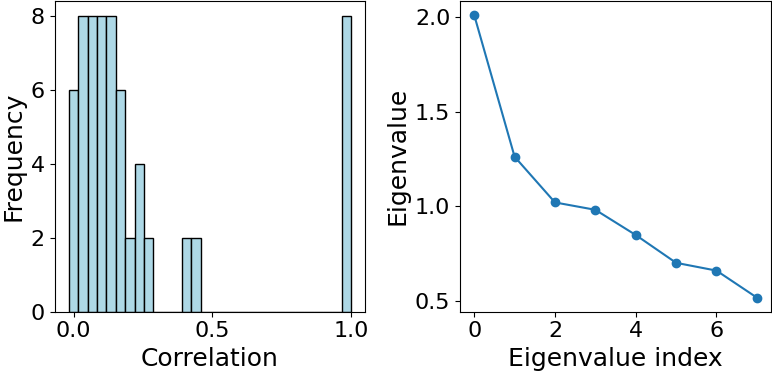}
        \subcaption{Wind data}
    \end{minipage}
    \begin{minipage}{0.32\textwidth}
        \includegraphics[width=\linewidth]{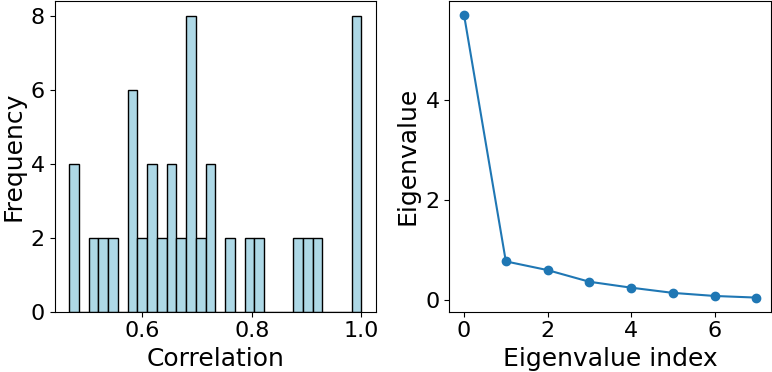}
        \subcaption{Solar data}
    \end{minipage}
    \begin{minipage}{0.32\textwidth}
        \includegraphics[width=\linewidth]{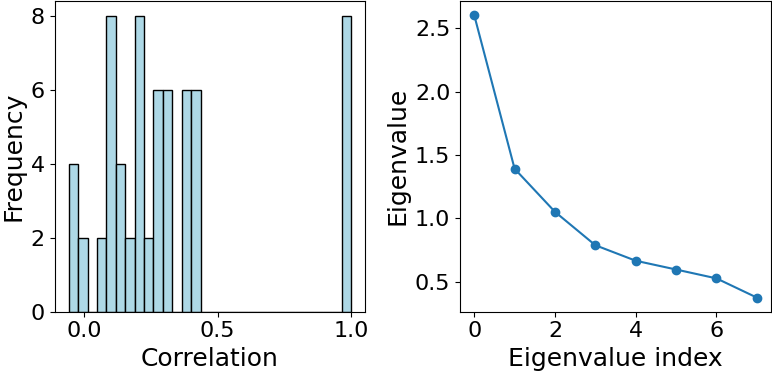}
        \subcaption{Traffic data}
    \end{minipage}
    \caption{Distribution of estimated correlation and the eigenvalues of corresponding correlation matrices on real-time series. We visualize the results using $p$-dimensional prediction residuals with $p=8$.}
    \label{fig:cov_residuals}
\end{figure}

Table \ref{tab:real_data} shows that the test coverage of \MultiSPCI{} and two CP methods is always valid by yielding coverage greater than or equal to the target 95\%. In contrast, the two probabilistic forecasting baselines may incur severe under-coverage, where TFT coverage is generally better than DeepAR's. Regarding the average size of prediction regions, we also note that the average size by \MultiSPCI{} is consistently smaller than those by baselines (except against Local ellipsoid on solar data when $p=2$), demonstrating that our proposed method quantifies prediction uncertainty more precisely. We believe these benefits come from using ellipsoidal rather than hyper-rectangular prediction sets and the adaptive re-estimation of quantiles of non-conformity scores.
Additionally, Table \ref{tab:high_d_tab} shows comparisons on higher-dimensional wind data, on which the benefits of \MultiSPCI{} persist.
Figure \ref{fig:real_data_rolling} further analyzes the rolling performance of different methods. We see that the rolling coverage of \MultiSPCI{} and the CP baselines all center around the target 95\% coverage value with reasonably small variations. Meanwhile, \MultiSPCI{} has a smaller rolling width than the CP baselines, indicating that our proposed method almost always yields smaller prediction regions.
Lastly, as seen in Figure \ref{fig:cov_residuals}, the estimated correlation between residuals from two different locations can be as high as 0.92 (see solar data). Thus, it is indeed necessary to consider such correlation when constructing prediction regions to quantify prediction uncertainty effectively.

\section{Discussions}

In this work, we proposed \MultiSPCI{}, a general sequential conformal prediction method for multivariate time series. Specifically, \MultiSPCI{} extends sequential univariate CP method to construct ellipsoids during test time. Extensions using local ellipsoids that are adaptive in shape are also discussed. Theoretically, we bound the coverage gap in finite samples without assuming data exchangeability. Empirically, on both simulation and real time series, we show \MultiSPCI{} yields significantly smaller prediction sets than baselines and maintains coverage. 

In the future, we will explore constructing prediction regions beyond ellipsoids. One such possibility is using convex hulls, which are irregular in shape but could lead to the tightest fit as prediction sets. We discuss such possibility and preliminary results in Appendix \ref{AppE}). We will also further study the theoretical properties of CP in high dimensions, leveraging existing results on multivariate quantile estimation.

\section*{Acknowledgement}

The authors would like to thank Jonghyeok Lee, and Bo Dai for their helpful discussions and comments. This work is partially supported by an NSF CAREER CCF-1650913, NSF DMS-2134037, CMMI-2015787, CMMI-2112533, DMS-1938106, DMS-1830210, and the Coca-Cola Foundation.

\bibliography{references}
\bibliographystyle{icml2024}

\appendix
\setcounter{table}{0}
\setcounter{figure}{0}
\setcounter{equation}{0}

\renewcommand{\thetable}{A.\arabic{table}}
\renewcommand{\thefigure}{A.\arabic{figure}}
\renewcommand{\theequation}{A.\arabic{equation}}
\renewcommand{\thealgorithm}{A.\arabic{algorithm}}

\section{Additional experiments}
\label{AppE}
\begin{figure}[!t]
    \centering
    \includegraphics[width=0.35\linewidth]{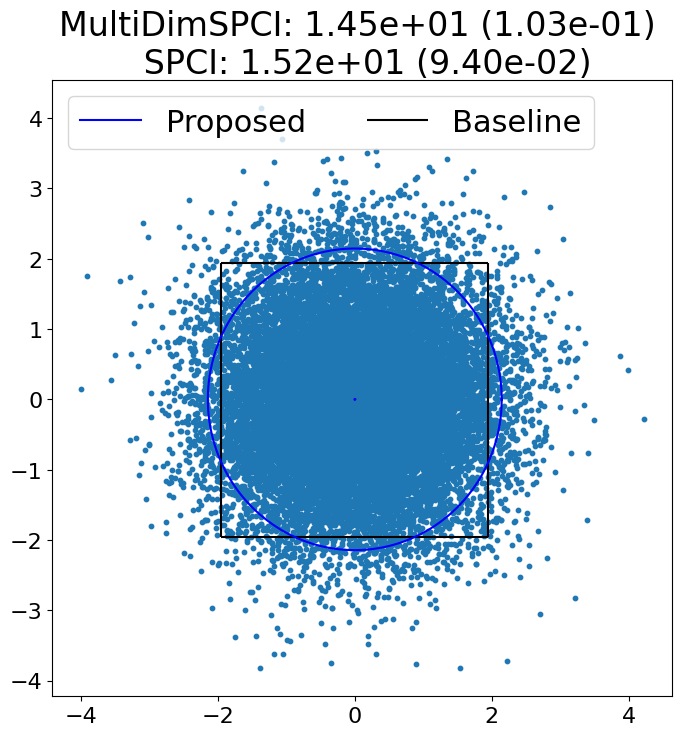}
    \includegraphics[width=0.35\linewidth]{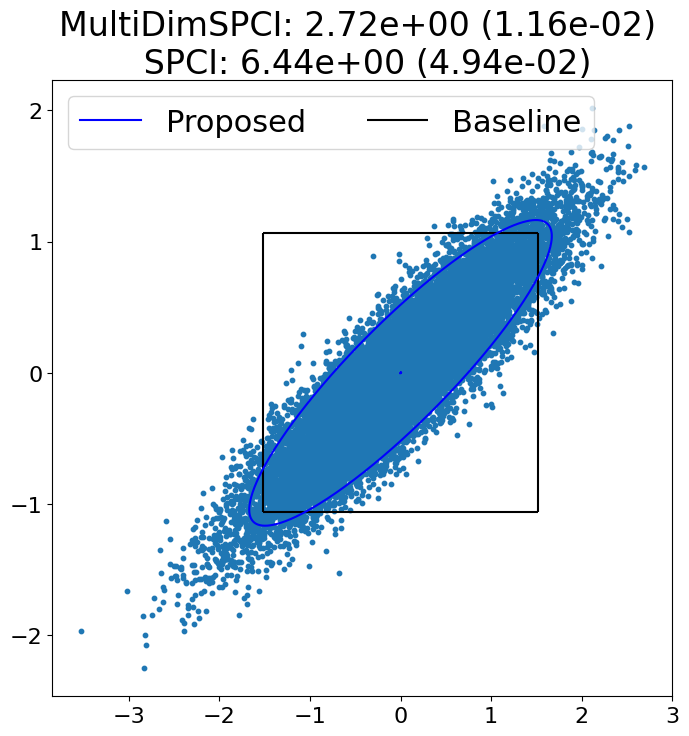}
    \caption{Non-critical regions on independent ${\rm AR}(w)$ (left) and ${\rm VAR}(w)$ (right) in $\R^2$. The average size of prediction regions is shown in captions. The size of the non-critical region by the proposed method is smaller, especially on ${\rm VAR}(w)$. As a result, the prediction regions by \MultiSPCI{} are smaller than those by \SPCI{}.}
    \label{fig:non_critical_region_simul}
\end{figure}
\paragraph{Comparison of non-critical regions.} From Figure \ref{fig:non_critical_region_simul}, we see that \MultiSPCI{} almost always yields smaller prediction sets than the coordinate-wise use of \SPCI{}, whose prediction regions are squares that tend to over-cover the test samples. In contrast, \MultiSPCI{} can well capture the dependency within $Y$ to enable accurate uncertainty quantification.

\begin{table}[!b]
    \centering
    \caption{Accuracy and size of the prediction sets on two independent $AR(w)$ sequences.}
    \label{tab:square}
    \resizebox{0.6\textwidth}{!}{
    \begin{tabular}{c|c|c|c|c}
    \hline
         & MultiDimSPCI &  Copula  & Convex hull & Baseline \\
    \hline
    Coverage & 90.2 & 90.6 & 90.0& 89.8 \\
    Size & 3.92 & 3.59 &3.64 & 3.58 \\
    \hline
    \end{tabular}}
\end{table}
\paragraph{Results using convex hulls.} 
Currently, the ellipsoid shape is utilized for the prediction region. This method is robust and guarantees coverage accuracy, as demonstrated by experiments. However, considering the distribution shape of the residuals may further enhance its performance. The distribution might not conform to an ellipsoidal shape in high-dimensional cases, potentially being irregular. As a result, the ellipsoidal shape may not be the most optimal or tight fit. What if we could allow our prediction set to adapt to any shape? In doing so, the new region would likely be much tighter in scenarios where the true residuals do not follow an ellipsoidal distribution.

In almost all instances, a convex hull can cover a set of points more compactly than an ellipsoid. We could achieve a significantly tighter fit by adopting a convex hull for the prediction region. 
The primary distinction between the two methods lies in the control parameters: the ellipsoid requires only the radius adjustment, whereas the convex hull necessitates control over all vertices. Ideally, we would select a set of data points that optimally balances coverage and minimizes the region size. However, this becomes computationally infeasible as the dataset size increases. Rather than optimizing the convex hull, we require it to cover exactly all the training data encompassed by the ellipsoid method.

The convex hull method selects the points in the training set that are covered by the ellipsoid method and uses the convex hull of these points as the prediction regions. It has a smaller region than the ellipsoid method because of how it is constructed.
As shown in Table \ref{tab:square}, the convex hull method on time series in $\mathbb{R}^2$ reaches valid coverage with a smaller prediction set. However, as shown in Table \ref{tab:square2}, getting a region with valid coverage in higher dimensions requires much more training data. The computational cost in higher dimensions becomes unaffordable if we want to reach a reasonable coverage. This aspect will be explored in future research.

\begin{table}[!t]
    \centering
    \caption{Accuracy and size of the prediction set for error uniformly spread in $[-1,1]^4$.}
    \label{tab:square2}
    \resizebox{0.6\textwidth}{!}{
    \begin{tabular}{c|c|c|c|c}
    \hline
         & MultiDimSPCI &  Copula  & Convex hull & Baseline \\
    \hline
    \makecell[c]{Coverage\\ ($80000$ training samples)} & 89.6 & 90.5 & 85.9& 89.9 \\
    \makecell[c]{Size\\ ($80000$ training samples)}& 22.2 & 14.5 &14.2 & 14.4 \\
    \makecell[c]{Coverage\\ ($800000$ training samples)} & 90.2 & 90.5 & 88.6& 89.9 \\
    \makecell[c]{Size\\ ($800000$ training samples)}& 22.2 & 14.5 &14.3 & 14.4\\
    \hline
    \end{tabular}}
\end{table}

\section{Proof}\label{AppP}

\begin{lemma}\label{lu}
$F_e(e_{T+1})\sim \mathrm{Unif}[0,1]$. 
\end{lemma} 

\begin{proof}
 This holds for random variable $e$ as long as the CDF $F_e$ is continuous and strictly increasing.    
\end{proof}

\begin{lemma}
Under Assumption \ref{a1}, for any training size T, there is an event $A_T$ which occurs
with probability at least $1-\sqrt{\log(16T)/T}$, such that conditioning on $A_T$,
\begin{equation}
\sup_x |\widetilde{F}_{T+1}(x)-F_e(x)|\le \sqrt{\frac{\log(16T)}{T}}.
\end{equation}
\end{lemma}

\begin{proof}
The proof follows the proof of Lemma $1$ in \citep{xu2023conformal}. When the error process is i.i.d., the famous Dvoretzky-Kiefer-Wolfowitz inequality \citep{kosorok2008introduction} implies that
\begin{equation}
\mathbb{P}\left(\sup _x|\widetilde{F}_{T+1}(x)-F_e(x)|>s_T\right) \leq 2 e^{-2 T s_T^2} .    
\end{equation}

Pick $s_T=\sqrt{W(16 T)}/(2 \sqrt{T})$, where $W(T)$ is the Lambert $W$ function that satisfies $W(T) e^{W(T)}=T$. We see that $s_T \leq \sqrt{\log (16 T) / T}$. Thus, define the event $A_T$ on which $\sup _x|\widetilde{F}_{T+1}(x)-F_e(x)| \leq$ $\sqrt{\log (16 T) / T}$, whereby we have
\begin{equation}
\begin{aligned}
\sup _x|\widetilde{F}_{T+1}(x)-F_e(x)|  \big\vert A_T &\leq \sqrt{\frac{\log(16T)}{T}},\\
P\left(A_T\right)&>1-\sqrt{\frac{\log(16T)}{T}}.
\end{aligned}
\end{equation}
\end{proof} 

\begin{lemma}[Theorem 1.2, \cite{vershynin2012close}]
\label{vershynin}
Consider a random vector $\varepsilon\in\mathbb{R}^p (p \geq 4)$ which has zero mean and satisfies moment assumption \ref{a4} for some $q>4$ and some $K_1, L$. Let $\delta>0$. Then, with probability at least $1-\delta$, the covariance matrix $\Sigma$ of $\varepsilon$ can be approximated by the sample covariance matrix $\frac{1}{T}\sum_{t=1}^{T}\varepsilon_t\otimes \varepsilon_t$ as
\begin{equation}
    \left\|\Sigma-\frac{1}{T}\sum_{t=1}^{T}\varepsilon_t\otimes \varepsilon_t\right\| \le C\left(\frac{1}{\delta}\right)^{20/9q}\log\left(\frac{1}{\delta}\right)(\log \log p)^2\left(\frac{p}{T}\right)^{1/2-2/q} ,
\end{equation}
where $C$ is a constant that depends only on parameters $q,K_1,L$.
\end{lemma} 

\begin{lemma}
\label{bound2}
Under Assumption \ref{a1}, \ref{a2}, \ref{a3} and \ref{a4}, with high probability $1-\delta$,
\begin{equation}
\sum_{t=1}^{T}|\hat{e}_t-e_t|\le \frac{T}{\lambda}\delta_T^2+ \frac{K_3 T}{\lambda^2}\left[C\left(\frac{1}{\delta}\right)^{20/9q}\log\left(\frac{1}{\delta}\right)(\log\log p)^2 \frac{p^{3/2-2/q}}{T^{1/2-2/q}} + 22K_3 \max\left\{\frac{p^3}{T\delta},p^{3/2}\delta_T\right\}\right],
\end{equation}
where $C$ is a constant that depends only on parameters $q,K_1,L$ and $K_3 = K_1 + \sqrt{3K_2}$.
\end{lemma}

\begin{proof}
For any test conformity score $\hat{e}$ and the corresponding $e$, we drop subscript $t$ here for notation simplicity.
\begin{equation}
\label{b41}
\begin{aligned}
|\hat{e}-e|&=|\hat{\varepsilon}^T\widehat \Sigma^{-1}\hat{\varepsilon}-\varepsilon^T\Sigma^{-1}\varepsilon|\\
&\le|\hat{\varepsilon}^T\widehat \Sigma^{-1}\hat{\varepsilon}-\varepsilon^T\widehat \Sigma^{-1}\varepsilon|+|\varepsilon^T\widehat \Sigma^{-1}\varepsilon-\varepsilon^T\Sigma^{-1}\varepsilon|\\
&\le \|\widehat \Sigma^{-1}\|\|\Delta\|^2+|\varepsilon^T(\widehat \Sigma^{-1}-\Sigma^{-1})\varepsilon|\\
&\le \frac{1}{\lambda}\|\Delta\|^2 + \|\varepsilon\|^2\|\widehat \Sigma^{-1}-\Sigma^{-1}\|\\
&=\frac{1}{\lambda}\|\Delta\|^2 + \|\varepsilon\|^2\|\widehat \Sigma^{-1}(\Sigma-\widehat \Sigma)\Sigma^{-1}\|\\
&\le\frac{1}{\lambda}\|\Delta\|^2 +\|\varepsilon\|^2\|\widehat \Sigma^{-1}\|\|\Sigma^{-1}\|\|\Sigma-\widehat \Sigma\|\\
&\le \frac{1}{\lambda}\|\Delta\|^2+\frac{1}{\lambda^2}\|\varepsilon\|^2\|\Sigma-\widehat \Sigma\|.
\end{aligned}
\end{equation}
Then the problem becomes bounding the spectral norm $\|\Sigma-\widehat \Sigma\|$. Recall that $\hat{\varepsilon}_t = \varepsilon_t + \Delta_t$, and we define $\bar{\varepsilon}_*=(\sum_{t=1}^{T}\hat{\varepsilon_t})/T$, $\bar{\Delta}=(\sum_{t=1}^{T}\Delta_t)/T$. The sample covariance matrix can be represented as
\begin{equation}
\begin{aligned}
\widehat \Sigma &= \frac{1}{T-1}\sum_{t=1}^T (\hat{\varepsilon}_t-\bar{\varepsilon})(\hat{\varepsilon}_t-\bar{\varepsilon})^T\\
&=\frac{1}{T-1}\sum_{t=1}^T [(\varepsilon_t+\Delta_t)-(\bar{\varepsilon}_{*}+\bar{\Delta})][(\varepsilon_t+\Delta_t)-(\bar{\varepsilon}_{*}+\bar{\Delta})]^T\\
&=\frac{1}{T-1}\sum_{t=1}^T [(\varepsilon_t-\bar{\varepsilon}_{*})+(\Delta_t-\bar{\Delta})][(\varepsilon_t-\bar{\varepsilon}_{*})+(\Delta_t-\bar{\Delta})]^T\\
&=\underbrace{\frac{1}{T-1}\sum_{t=1}^T(\varepsilon_t-\bar{\varepsilon}_{*})(\varepsilon_t-\bar{\varepsilon}_{*})^T}_{\circled{1}} + \underbrace{\frac{1}{T-1}\sum_{t=1}^T(\Delta_t-\bar{\Delta})(\Delta_t-\bar{\Delta})^T}_{\circled{2}} + \\
&\quad \underbrace{\frac{1}{T-1}\sum_{t=1}^T\left[(\varepsilon_t-\bar{\varepsilon}_{*})(\Delta_t-\bar{\Delta})^T+(\Delta_t-\bar{\Delta})(\varepsilon_t-\bar{\varepsilon}_{*})^T\right]}_{\circled{3}}.
\end{aligned}
\end{equation}
The first term is the sample covariance matrix of $\varepsilon_t$, which typically converges to the true covariance matrix when the dimension $p$ is negligible compared to $T$. From the assumption \ref{a2}, the magnitude of the second term and the third term should be bounded by $\delta_T$, which is the accuracy of prediction. This means that
\begin{equation}
\begin{aligned}
\|\widehat \Sigma-\Sigma\|&\le \|\circled{1}-\Sigma\| + \|\circled{2}\|+\|\circled{3}\|.
\end{aligned}
\end{equation}
We can bound each spectral norm respectively.
\begin{align*}
\|\circled{1} - \Sigma\| &=  \left\|\frac{1}{T-1}\left(\sum_{t=1}^T \varepsilon_t\varepsilon_{t}^{T}-T\bar{\varepsilon}_{*}\bar{\varepsilon}_{*}^T\right)-\Sigma\right\|\\
& = \left\|\left(\frac{1}{T}\sum_{t=1}^T \varepsilon_t\varepsilon_{t}^{T}-\Sigma\right) +\frac{1}{T(T-1)}\left(\sum_{t=1}^T  \varepsilon_t\varepsilon_{t}^{T} - T^2\bar{\varepsilon}_{*}\bar{\varepsilon}_{*}^T\right)\right\|\\
&\le \left\|\frac{1}{T}\sum_{t=1}^T \varepsilon_t\varepsilon_{t}^{T}-\Sigma\right\| + \frac{1}{T(T-1)}\left(\left\|\sum_{t=1}^T  \varepsilon_t\varepsilon_{t}^{T}\right\|+\left\|T^2\bar{\varepsilon}_{*}\bar{\varepsilon}_{*}^T\right\|\right)\\
& \stackrel{(i)}{\leq} C \left(\frac{3}{\delta}\right)^{\frac{20}{9q}}\log \left(\frac{3}{\delta}\right)(\log\log p)^2 \left(\frac{p}{T}\right)^{\frac{1}{2}-\frac{2}{q}} + \frac{1}{T(T-1)}\left(\sum_{t=1}^T  \|\varepsilon_t\varepsilon_{t}^{T}\|+T^2\|\bar{\varepsilon}_{*}\bar{\varepsilon}_{*}^T\|\right)\\
&\stackrel{(ii)}{\leq} 4C \left(\frac{1}{\delta}\right)^{\frac{20}{9q}}\log \left(\frac{1}{\delta}\right)(\log\log p)^2 \left(\frac{p}{T}\right)^{\frac{1}{2}-\frac{2}{q}} + \frac{1}{T(T-1)}\sum_{t=1}^T  \|\varepsilon_t\|^2 + 2\|\bar{\varepsilon}_*\|^2,
\end{align*}
where $(i)$ holds with high probability $1-\frac{\delta}{3}$ under Assumption \ref{a4} according to the Lemma \ref{vershynin} from Theorem $1.2$ in \cite{vershynin2012close} and $(ii)$ holds because $3^{20/9q}\le 2$, $\log(3/\delta)\le 2\log(1/\delta)$ when $\delta\le 1/3$ and $T/(T-1)\le 2$. We can put the constant $4$ into the constant $C$ which simplifies the notation. From now on, we use $C$ to represent the whole constant $4C$ in the expression. For the other term $\|\bar{\varepsilon}_*\|^2=\|(\sum_{t=1}^{T}\varepsilon_t)/T\|^2$, we can bound it with Chebshev's inequality. Since $\varepsilon_t \in\R^p$, we use $\varepsilon_{ti} (1\le i\le p)$ to denote the $i$th entry of random vector $\varepsilon_t$.
\begin{equation}
\begin{aligned}
\|\bar{\varepsilon}_*\|^2 = \left\|\frac{\sum_{t=1}^{T}\varepsilon_t}{T}\right\|^2 = \sum_{i=1}^p \left|\frac{\sum_{t=1}^{T}\varepsilon_{ti}}{T}\right|^2.
\end{aligned}
\end{equation}

Using Chebshev inequality, we have that
\begin{equation}
\begin{aligned}
\p\left(\left|\frac{\sum_{t=1}^{T}\varepsilon_{ti}}{T}\right| \ge \sqrt{\frac{3p\mathrm{Var}(\varepsilon_{ti})}{T\delta}}\right)&\le \frac{\mathrm{Var}(\varepsilon_{ti})}{T(\frac{3p\mathrm{Var}(\varepsilon_{ti})}{T\delta})}=\frac{\delta}{3p}
\end{aligned}
\end{equation}

This means that
\begin{equation}
\begin{aligned}
\p\left(\|\bar{\varepsilon}_*\|^2\le \frac{3p\sum_{i=1}^{p}\mathrm{Var}(\varepsilon_{ti})}{T\delta}\right)&=\p\left(\sum_{i=1}^p \left|\frac{\sum_{t=1}^{T}\varepsilon_{ti}}{T}\right|^2\le \frac{3p\sum_{i=1}^{p}\mathrm{Var}(\varepsilon_{ti})}{T\delta}\right)\\
&\ge \p\left(\bigcap_{i=1}^{p}\Bigg\{\left|\frac{\sum_{t=1}^{T}\varepsilon_{ti}}{T}\right|^2\le \frac{3p\mathrm{Var}(\varepsilon_{ti})}{T\delta}\Bigg\}\right)\\
&\ge \left(1-\frac{\delta}{3p}\right)^p \ge 1-\frac{\delta}{3}.
\end{aligned}
\end{equation}

Considering that \[\sum_{i=1}^{p}\mathrm{Var}(\varepsilon_{ti}) = \mathbb{E}(\|\varepsilon_t\|^2) \le K_1 p. \]

We have that with probability higher than $1-\delta/3$,
\begin{equation}
\begin{aligned}
\|\circled{1} - \Sigma\| &\le C \left(\frac{1}{\delta}\right)^{\frac{20}{9q}}\log \left(\frac{1}{\delta}\right)(\log\log p)^2 \left(\frac{p}{T}\right)^{\frac{1}{2}-\frac{2}{q}} + \frac{1}{T(T-1)}\sum_{t=1}^T  \|\varepsilon_t\|^2 + 2\|\bar{\varepsilon}_*\|^2 \\
&\le C \left(\frac{1}{\delta}\right)^{\frac{20}{9q}}\log \left(\frac{1}{\delta}\right)(\log\log p)^2 \left(\frac{p}{T}\right)^{\frac{1}{2}-\frac{2}{q}} + \frac{1}{T(T-1)}\sum_{t=1}^T  \|\varepsilon_t\|^2 + \frac{6K_1 p^2}{T\delta}.
\end{aligned}
\end{equation}

For the second term, we have
\begin{equation}
\begin{aligned}
\|\circled{2}\| &= \left\|\frac{1}{T-1}\left(\sum_{t=1}^{T}\Delta_t\Delta_{t}^{T}-T\bar{\Delta}\bar{\Delta}^T\right)\right\|\\
&\le \frac{1}{T-1}\left(\left\|\sum_{t=1}^{T}\Delta_t\Delta_{t}^{T}\right\| +T\|\bar{\Delta}\bar{\Delta}^T\|\right)\\
&\le \frac{2}{T-1}\sum_{t=1}^{T}\|\Delta_t\|^2\\
&\le \frac{2T\delta_{T}^{2}}{T-1}.
\end{aligned}
\end{equation}
For the third term, we have
\begin{align*}
\|\circled{3}\| &= \frac{1}{T-1}\left\|\sum_{t=1}^{T}(\varepsilon_t\Delta_{t}^{T}+\Delta_{t}\varepsilon_{t}^T)-T(\bar{\varepsilon}_{*}\bar{\Delta}^T+\bar{\Delta}\bar{\varepsilon}_{*}^T)\right\|\\
&\le \frac{2}{T-1}\left(\left\|\sum_{t=1}^{T}\varepsilon_t\Delta_{t}^{T}\right\|+T\|\bar{\Delta}\|\|\bar{\varepsilon}_{*}\|\right)\\
&= \frac{2}{T-1}\left(\left\|\sum_{t=1}^{T}\varepsilon_t\Delta_{t}^{T}\right\|+\sqrt{\left(\frac{\|\sum_{t=1}^T\varepsilon_t\|^2}{T}\right)\left(\frac{\|\sum_{t=1}^T\Delta_t\|^2}{T}\right)}\right)\\
&\le \frac{2}{T-1}\left(\sum_{t=1}^{T}\|\varepsilon_t\|\|\Delta_{t}\|+\sqrt{\left(\sum_{t=1}^T\|\varepsilon_t\|^2\right)\left(\sum_{t=1}^T\|\Delta_t\|^2\right)}\right)\\
& \stackrel{(i)}{\le} \frac{4}{T-1}\sqrt{\left(\sum_{t=1}^T\|\varepsilon_t\|^2\right)\left(\sum_{t=1}^T\|\Delta_t\|^2\right)}\\
&\le \frac{4\delta_T}{T-1}\sqrt{T\sum_{t=1}^T\|\varepsilon_t\|^2}.
\end{align*}
The inequality $(i)$ holds because of Cauchy-Schwarz inequality $(\sum_{t=1}^T \|\varepsilon_t\|^2)(\sum_{t=1}^T \|\Delta_t\|^2)\ge (\sum_{t=1}^T \|\varepsilon_t\|\|\Delta_{t}\|)^2$. From Assumption \ref{a4}, we have
\begin{equation}
\begin{aligned}
\mathbb{E}\left[\sum_{t=1}^T\|\varepsilon_t\|^2\right] = T \mathbb{E}(\|\varepsilon_t\|^2)\le TK_1p.
\end{aligned}
\end{equation}
Using Chebshev's inequality we have
\begin{equation}
\begin{aligned}
\p\left\{\left|\frac{1}{T}\sum_{t=1}^{T}\|\varepsilon_t\|^2-\mathbb{E}[\|\varepsilon_t\|^2]\right|\ge \sqrt{\frac{3\mathrm{Var}[\|\varepsilon_t\|^2]}{T\delta}}\right\}\le \frac{\mathrm{Var}[\|\varepsilon_t\|^2]}{T(\frac{3\mathrm{Var}[\|\varepsilon_t\|^2])}{T\delta})}=\frac{\delta}{3},
\end{aligned}
\end{equation}
which means with probability higher than $1-\delta/3$,
\begin{equation}
\label{bound2m}
\begin{aligned}
\frac{1}{T}\sum_{t=1}^{T}\|\varepsilon_t\|^2&\le \mathbb{E}[\|\varepsilon_t\|^2] + \sqrt{\frac{3\mathrm{Var}[\|\varepsilon_t\|^2]}{T\delta}}\\
&\le K_1p+\sqrt{\frac{3K_2p}{T\delta}}\\
&\le (K_1+\sqrt{3K_2})p\coloneqq K_3p.
\end{aligned}
\end{equation}
The last inequality holds because of Assumption \ref{a4} and $pT\delta\ge 1$. Without loss of generality, we can assume $K_3\ge 1$. Define $S_T=\frac{1}{T}\sum_{t=1}^{T}\|\varepsilon_t\|^2$. Overall, with probability higher than $1-\delta$, we have inequality \ref{bound2m} and the following inequality holds when $T\ge p$
\begin{align*}
\|\widehat \Sigma-\Sigma\|\le C\left(\frac{1}{\delta}\right)^{\frac{20}{9q}}\log\left(\frac{1}{\delta}\right)(\log\log p)^2 (\frac{p}{T})^{\frac{1}{2}-\frac{2}{q}} &+ \frac{1}{T(T-1)}\sum_{t=1}^T  \|\varepsilon_t\|^2 \\
+\frac{6K_1 p^2}{T\delta}+ \frac{2T\delta_{T}^{2}}{T-1} &+ \frac{4\delta_T}{T-1}\sqrt{T\sum_{t=1}^T\|\varepsilon_t\|^2}\\
= 2C\left(\frac{1}{\delta}\right)^{\frac{20}{9q}}\log\left(\frac{3}{\delta}\right)(\log\log p)^2 (\frac{p}{T})^{\frac{1}{2}-\frac{2}{q}} &+ \frac{2}{T-1}S_T +\frac{6K_1 p^2}{T\delta}
+ \frac{2T\delta_{T}^{2}}{T-1} + \frac{4T\delta_T}{T-1}\sqrt{S_T
}\\
\stackrel{(i)}{\le} C\left(\frac{1}{\delta}\right)^{\frac{20}{9q}}\log\left(\frac{1}{\delta}\right)(\log\log p)^2 (\frac{p}{T})^{\frac{1}{2}-\frac{2}{q}}&+\frac{4}{T}S_T  +\frac{6K_1 p^2}{T\delta}
+ 4\delta_{T}^{2} + 8\delta_T\sqrt{S_T}\\
\le  C\left(\frac{1}{\delta}\right)^{\frac{20}{9q}}\log\left(\frac{1}{\delta}\right)(\log\log p)^2 (\frac{p}{T})^{\frac{1}{2}-\frac{2}{q}}&+16\max\left\{\frac{S_T}{T},\delta_T^2,\delta_T\sqrt{S_T}\right\} +\frac{6K_1 p^2}{T\delta}\\
\le  C\left(\frac{1}{\delta}\right)^{\frac{20}{9q}}\log\left(\frac{1}{\delta}\right)(\log\log p)^2 (\frac{p}{T})^{\frac{1}{2}-\frac{2}{q}}& + 16\max\left\{\frac{K_3p}{T},\delta_T^2,\sqrt{K_3p}\delta_T\right\} +\frac{6K_3 p^2}{T\delta}\\
\le  C\left(\frac{1}{\delta}\right)^{\frac{20}{9q}}\log\left(\frac{1}{\delta}\right)(\log\log p)^2 (\frac{p}{T})^{\frac{1}{2}-\frac{2}{q}}& + 22K_3\max\left\{\frac{p}{T},\delta_T^2,\sqrt{p}\delta_T,\frac{p^2}{T\delta}\right\}\\
\stackrel{(ii)}{\le}  C\left(\frac{1}{\delta}\right)^{\frac{20}{9q}}\log\left(\frac{1}{\delta}\right)(\log\log p)^2 (\frac{p}{T})^{\frac{1}{2}-\frac{2}{q}}& + 22K_3 \max\left\{\frac{p^2}{T\delta},\sqrt{p}\delta_T\right\}.
\end{align*}
where $(i)$ holds because $\frac{T}{T-1}\le 2$ and $(ii)$ holds because $p\ge \delta_{T}^2$. Then the following inequality holds with high probability $1-\delta$,
\begin{equation}
\label{b42}
\begin{aligned}
\sum_{t=1}^{T}|\hat{e}_t-e_t|&\le \frac{1}{\lambda}\sum_{t=1}^{T}\|\Delta_t\|^2+\frac{1}{\lambda^2}\sum_{t=1}^{T}\|\varepsilon_t\|^2\|\widehat \Sigma-\Sigma\|  \\
&\le\frac{T}{\lambda}\delta_T^2+ \frac{K_3 Tp}{\lambda^2}\left[C\left(\frac{1}{\delta}\right)^{\frac{20}{9q}}\log\left(\frac{1}{\delta}\right)(\log\log p)^2 (\frac{p}{T})^{\frac{1}{2}-\frac{2}{q}} + 22K_3 \max\left\{\frac{p^2}{T\delta},\sqrt{p}\delta_T\right\}\right]\\
& = \frac{T}{\lambda}\delta_T^2+ \frac{K_3 T}{\lambda^2}\left[C\left(\frac{1}{\delta}\right)^{\frac{20}{9q}}\log\left(\frac{1}{\delta}\right)(\log\log p)^2 \frac{p^{3/2-2/q}}{T^{1/2-2/q}} + 22K_3 \max\left\{\frac{p^3}{T\delta},p^{3/2}\delta_T\right\}\right].
\end{aligned}
\end{equation}
\end{proof}

\begin{corollary}
\label{b5}
If the true covariance matrix $\Sigma$ is known and we use $\widehat \Sigma=\Sigma$, then 
\begin{equation}
\sum_{t=1}^{T}|\hat{e}_t-e_t|\le \frac{T}{\lambda}\delta_T^2.
\end{equation}
\end{corollary}

\begin{proof}
This is because when we use $\hat{\Sigma}=\Sigma$, the second term in Equation \ref{b42} is zero. Then the bound is simply the first term.
\end{proof}

\begin{lemma}
Under Assumption \ref{a1}, \ref{a2}, \ref{a3} and \ref{a4}, with a high probability $1-\delta$,
\begin{equation}
\sup_x |\widehat{F}_{T+1}(x)-\widetilde{F}_{T+1}(x)|\le(L_{T+1}+1) C_S+2 \sup _x|\widetilde{F}_{T+1}(x)-F_e(x)|,
\end{equation}
where \[C_S = \left(\frac{\delta_T^2}{\lambda}+ \frac{K_3}{\lambda^2}\left[C\left(\frac{1}{\delta}\right)^{\frac{20}{9q}}\log\left(\frac{1}{\delta}\right)(\log\log p)^2 \frac{p^{3/2-2/q}}{T^{1/2-2/q}} + 22K_3 \max\left\{\frac{p^3}{T\delta},p^{3/2}\delta_T\right\}\right]\right)^{1/2} \] and $K_3 = K_1+\sqrt{3K_2}$.
\end{lemma}

\begin{proof}
Using lemma \ref{bound2} and \ref{a2}, we have that with probability $1-\delta$
\begin{equation}
\begin{aligned}
\sum_{t=1}^{T}|e_t - \hat{e}_t|&\le \frac{T}{\lambda}\delta_T^2+ \frac{K_3 T}{\lambda^2}\left[C\left(\frac{1}{\delta}\right)^{20/9q}\log\left(\frac{1}{\delta}\right)(\log\log p)^2 \frac{p^{3/2-2/q}}{T^{1/2-2/q}} + 22K_3 \max\left\{\frac{p^3}{T\delta},p^{3/2}\delta_T\right\}\right] \\
&=TC_{S}^{2}.
\end{aligned}
\end{equation}

Let $S=\{t:|e_t-\hat{e}_t|\ge C_S\}$. Then
\begin{equation}
|S|C_S\le\sum_{t=1}^{T}|e_t - \hat{e}_t|\le TC_{S}^{2}.
\end{equation}
So $|S|\le  TC_S$. Then
\begin{align*}
 |\widehat{F}_{T+1}(x)-\widetilde{F}_{T+1}(x)| &\le \frac{1}{T} \sum_{t=1}^T|\mathbf{1}\{\hat{e}_t \leq x\}-\mathbf{1}\{e_t \leq x\}| \\
&\le \frac{1}{T}\left(|S|+\sum_{t \notin S}|\mathbf{1}\{\hat{e}_t \leq x\}-\mathbf{1}\{e_t \leq x\}|\right) \\
&\stackrel{(i)}{\leq} \frac{1}{T}\left(|S|+\sum_{t \notin S} \mathbf{1}\{|e_t-x| \leq C_S\}\right) \\
&\le \frac{1}{T}\left(|S|+\sum_{t=1}^T \mathbf{1}\{|e_t-x| \leq C_S\}\right) \\
&\le C_S+\mathbb{P}(|e_{T+1}-x| \leq C_S) + \\
& \quad \sup_x\left|\frac{1}{T} \sum_{t=1}^T \mathbf{1}\{|e_t-x| \leq C_S\}-\mathbb{P}(|e_{T+1}-x| \leq C_S)\right| \\
&= C_S+[F_e(x+C_S)-F_e(x-C_S)]+\sup _x \Big\vert [\widetilde{F}_{T+1}(x+C_S)-\widetilde{F}_{T+1}(x-C_S)]\\
& \quad - [F_e(x+C_S)-F_e(x-C_S)] \Big\vert \\
&\stackrel{(ii)}{\leq} (L_{T+1}+1) C_S+2 \sup _x|\widetilde{F}_{T+1}(x)-F_e(x)|,
\end{align*}
where $(i)$ is because $|\mathbf{1}\{a \leq x\}-\mathbf{1}\{b \leq x\}|\le \mathbf{1}\{|b-x|\le|a-b|\}$ for $a, b \in \R$ and $(ii)$ is because the Lipschitz continuity of $F_e(x)$.
\end{proof}

\begin{corollary}
If the true covariance matrix $\Sigma$ is known and we use $\widehat \Sigma=\Sigma$. Under Assumption \ref{a1}, \ref{a2}, \ref{a3} and \ref{a4}, with a high probability $1-\delta$,
\begin{equation}
\sup_x |\widehat{F}_{T+1}(x)-\widetilde{F}_{T+1}(x)|\le(L_{T+1}+1)\frac{\delta_T}{\sqrt{\lambda}} +2 \sup _x|\widetilde{F}_{T+1}(x)-F_e(x)|.
\end{equation}
\end{corollary}

\begin{theorem}
\label{b8}
Under Assumption \ref{a1}, \ref{a2}, \ref{a3} and \ref{a4}, with a high probability $1-\delta$, for any training size $T$ and $\alpha\in(0,1)$, we have
\begin{equation}
\begin{aligned}
&|\p(Y_{T+1} \in \widehat{C}_{T+1}^\alpha \mid X_{T+1}=x_{T+1})-(1-\alpha)| \\
&\leq 12 \sqrt{\log (16 T) / T}+4(L_{T+1}+1)(C_S+\delta_T).
\end{aligned}
\end{equation}
\end{theorem}
\begin{proof}
First, we recall the notation here. We define $\hat{\varepsilon}_t=y_t-\hat{f}(x_t)$ as the prediction residual, $\hat{e}_t=\hat{\varepsilon}_t^T\widehat \Sigma^{-1}\hat{\varepsilon}_t$ as the non-conformity score, $e_t=\varepsilon_t^T\Sigma^{-1}\varepsilon_t$ and $\Delta_t=\hat{\varepsilon}_t-\varepsilon_t$. Besides, we define the empirical CDF
\begin{equation}
\begin{aligned}
\widehat{F}_{T+1}(x)&=\frac{1}{T}\sum_{t=1}^{T}1\{\hat{e}_t\le z\}\\
\widetilde{F}_{T+1}(x)&=\frac{1}{T}\sum_{t=1}^{T}1\{e_t\le z\}.
\end{aligned}
\end{equation}
For any $\beta \in[0, \alpha]$, following the idea in Section \ref{TA}:
\begin{equation}
\begin{aligned}
& \left|\mathbb{P}\left(Y_{T+1} \in \widehat{C}_{T+1}^\alpha \big\vert X_{T+1}=x_{T+1}\right)-(1-\alpha)\right| \\
= & \left|\mathbb{P}\left(\hat{e}_{T+1} \in [\beta\text{ quantile of }\widehat{F}_{T+1},1-\alpha+\beta\text{ quantile of }\widehat{F}_{T+1}] \big\vert X_{T+1}=x_{T+1}\right)-(1-\alpha)\right| \\
= & \left|\mathbb{P}\left(\beta \leq \widehat{F}_{T+1}\left(\hat{e}_{T+1}\right) \leq 1-\alpha+\beta\right)-  \mathbb{P}\left(\beta \leq F_e\left(e_{T+1}\right) \leq 1-\alpha+\beta\right)\right|.
\end{aligned}
\label{e1}
\end{equation}
The last equality is a result of Lemma \ref{lu}. We can further rewrite the equation \ref{e1} as follows:
\begin{equation}
\begin{aligned}
&|\mathbb{P}(\beta \leq \widehat{F}_{T+1}(\hat{e}_{T+1}) \leq 1-\alpha+\beta) - \mathbb{P}(\beta \leq F_e(e_{T+1}) \leq 1-\alpha+\beta)| \\
& \leq \mathbb{E} | \mathbf{1}\{\beta \leq \widehat{F}_{T+1}(\hat{e}_{T+1}) \leq 1-\alpha+\beta\} - \mathbf{1}\{\beta \leq F_e(e_{T+1}) \leq 1-\alpha+\beta\} | \\
& \leq \mathbb{E}(|\mathbf{1}\{\beta \leq \widehat{F}_{T+1}(\hat{e}_{T+1})\} - \mathbf{1}\{\beta \leq F_e(e_{T+1})\}| + \\
&\quad |\mathbf{1}\{\widehat{F}_{T+1}(\hat{e}_{T+1}) \leq 1-\alpha+\beta\} - \mathbf{1}\{F_e(e_{T+1}) \leq 1-\alpha+\beta\}|).
\end{aligned}    
\end{equation}
The last inequality is because that for any constants $a, b$ and univariates $x, y$, \[|\mathbf{1}\{a \leq x \leq b\}-1\{a \leq y \leq b\}|\leq| \mathbf{1}\{a \leq x\}-\mathbf{1}\{a \leq y\}|+| \mathbf{1}\{x \leq b\}-\mathbf{1}\{y \leq b\}|.\] 
Then we have 
\begin{align*}
\mathbb{E}|\mathbf{1}\{\beta \leq \widehat{F}_{T+1}(\hat{e}_{T+1})\} - \mathbf{1}\{\beta \leq F_e(e_{T+1})\}|\le \mathbb{P}(|F_e(e_{T+1})-\beta| &\leq |\widehat{F}_{T+1}(\hat{e}_{T+1})-F_e(e_{T+1})|)\\
\mathbb{E}|\mathbf{1}\{\widehat{F}_{T+1}(\hat{e}_{T+1}) \leq 1-\alpha+\beta\} - \mathbf{1}\{F_e(e_{T+1}) &\leq 1-\alpha+\beta\}|\le \\
\mathbb{P}(|F_e(e_{T+1})-(1-\alpha+\beta)| \leq |\widehat{F}_{T+1}(\hat{e}_{T+1})-&F_e(e_{T+1})|),
\end{align*}
which holds since $|\mathbf{1}\{a \leq x\}-\mathbf{1}\{b \leq x\}| \leq \mathbf{1}\{|b-x| \leq|a-b|\}$ for any constant $a, b$ and univariate $x$ and $\mathbb{E}[\mathbf{1}\{A\}]=\mathbb{P}(A)$. Recall in Lemma \ref{l1}, we defined \(A_T\) as the event on which
\[
\sup_x |\widetilde{F}_{T+1}(x) - F_e(x)| \big\vert A_T \leq \sqrt{\frac{\log (16T)}{T}},
\]
where \(\mathbb{P}(A_T) > 1 - \sqrt{\frac{\log (16T)}{T}}\). Let \(A_T^C\) denote the complement of the event \(A_T\). For any \(\gamma \in [0,1]\), we have that
\begin{equation}
\begin{aligned}
&\mathbb{P}(|F_e(e_{T+1})-\gamma| \leq |\widehat{F}_{T+1}(\hat{e}_{T+1})-F_e(e_{T+1})|) \\
&\leq \mathbb{P}(|F_e(e_{T+1})-\gamma| \leq |\widehat{F}_{T+1}(\hat{e}_{T+1})-F_e(e_{T+1})| \big\vert A_T) + \mathbb{P}(A_T^C) \\
&\leq \mathbb{P}(|F_e(e_{T+1})-\gamma| \leq |\widehat{F}_{T+1}(\hat{e}_{T+1})-F_e(e_{T+1})| \big\vert A_T) + \mathbb{P}(A_T^C) \\
&\leq \mathbb{P}(|F_e(e_{T+1})-\gamma| \leq |\widehat{F}_{T+1}(\hat{e}_{T+1})-F_e(\hat{e}_{T+1})| + |F_e(\hat{e}_{T+1})-F_e(e_{T+1})| \big\vert A_T) + \sqrt{\frac{\log (16T)}{T}} .
\end{aligned}
\end{equation}
To bound the conditional probability above, we note that with a high probability $1-\delta$, conditioning on the event $A_T$,
\begin{equation}
\begin{aligned}
& |\widehat{F}_{T+1}(\hat{e}_{T+1})-F_e(\hat{e}_{T+1})| + |F_e(\hat{e}_{T+1})-F_e(e_{T+1})| \big\vert A_T \\
& \leq \sup_x |\widehat{F}_{T+1}(x)-F_e(x)| \big\vert A_T + L_{T+1}| \hat{e}_{T+1} - e_{T+1} | \\
& \leq \sup_x |\widehat{F}_{T+1}(x)-\widetilde{F}_{T+1}(x)| \big\vert A_T + \sup_x | \widetilde{F}_{T+1}(x)-F_e(x)| \big\vert A_T + L_{T+1} |\hat{e}_{T+1}-e_{T+1}| \\
& \leq (L_{T+1}+1) C_S + 3 \sup_x |\widetilde{F}_{T+1}(x)-F_e(x)| \big\vert A_T + L_{T+1} \delta_T \\
& \leq 3 \sqrt{\frac{\log(16T)}{T}} + (L_{T+1}+1)(C_S+\delta_T).
\end{aligned}    
\end{equation}
which follows the result from Lemma \ref{l1} and \ref{l2}.

Therefore, because $F_e(e_{T+1}) \sim \mathrm{Unif}[0,1]$, we have
\begin{equation}
\begin{aligned}
& \mathbb{P}\left(|F_e(e_{T+1})-\gamma| \leq |\widehat{F}_{T+1}(\hat{e}_{T+1})-F_e(\hat{e}_{T+1})| + |F_e(\hat{e}_{T+1})-F_e(e_{T+1})| \big\vert A_T\right) \\
\leq & 6 \sqrt{\frac{\log(16T)}{T}} + 2(L_{T+1}+1)(C_S+\delta_T) .
\end{aligned}    
\end{equation}

As a result, by letting $\gamma=\beta$ and $1-\alpha+\beta$, we have
\begin{equation}
\begin{aligned}
|\mathbb{P}(Y_{T+1} \in \widehat{C}_{T+1}^\alpha \mid X_{T+1}=x_{T+1})-(1-\alpha)| \\
\leq 12 \sqrt{\frac{\log(16T)}{T}} + 4(L_{T+1}+1)(C_S+\delta_T).
\end{aligned}
\end{equation}
\end{proof}

Now we have an asymptotic coverage guarantee for the case where $\{\varepsilon_t\}_{t=1}^{T}$ is i.i.d., and we can extend the result to the case where $\{\varepsilon_t\}_{t=1}^{T}$ is stationary and strong mixing.
\begin{assumption}
Assume $\{\varepsilon_t\}_{t=1}^{T+1}$ is stationary and strong mixing with the mixing coefficients $\sum_{k>0}\alpha_k<M$. Meanwhile, $F_e(x)$ (the CDF of true non-conformity score) is Lipschitz continuous with constant $L_{T+1}>0$.
\label{a5}
\end{assumption}

The properties of stationary and strong mixing can be imparted to the sequence $\{e_t\}_{t=1}^{T}$.
\begin{lemma}
Under Assumption \ref{a5}, $\{e_t\}_{t=1}^{T}$ is stationary and strong mixing with coefficients $\sum_{k>0}\alpha_k(\{e_t\}_{t\ge1})<M$, where $\alpha_k(\{e_t\}_{t\ge1})$ represents the $\alpha-$mixing coefficients of the random sequence $\{e_t\}_{t\ge 1}$.
\end{lemma}
\begin{proof}
The relationship between $e_t$ and $\varepsilon_t$ is that
\begin{equation}
e_t = \varepsilon_{t}^{T}\Sigma^{-1}\varepsilon_{t}.
\end{equation}
Define $f(x)=x^T\Sigma^{-1}x$. Since $f(x)$ is a Borel-measurable function, we have that $f^{-1}(B)$ is also Borel-measurable for any Borel set $B\subset\R$. Thus we have
\begin{equation}
\begin{aligned}
\p(e_{n_1}\in B_{n_1},\cdots,e_{n_k}\in B_{n_k})&=\p(\varepsilon_{n_1}\in f^{-1}(B_{n_1}),\cdots,\varepsilon_{n_k}\in f^{-1}(B_{n_k}))\\
&=\p(\varepsilon_{n_1+h}\in f^{-1}(B_{n_1}),\cdots,\varepsilon_{n_k+h}\in f^{-1}(B_{n_k}))\\
&=\p(e_{n_1+h}\in B_{n_1},\cdots,e_{n_k+h}\in B_{n_k}),
\end{aligned}
\end{equation}
which shows the stationarity of $\{e_t\}_{t=1}^{T}$. Besides, the $\sigma$-algebra generated by $f(\varepsilon_t)$ is contained in the $\sigma$-algebra generated by $\varepsilon_t$; consequently for all $I \subset \mathbb{Z}$ (possibly infinite),
\begin{equation}
\sigma(f(\varepsilon_t), t \in I) \subset \sigma(\varepsilon_t, t \in I).
\end{equation}

Since the definition of the mixing coefficient is the maximum over the sub-sigma algebra generated by the sequence, it follows that for all $k$,
\begin{equation}
\alpha_k(\{f(\varepsilon_t)\}_{t \ge 1}) \le\alpha_k.
\end{equation}
\end{proof}
As a result, we have that $\{e_t\}_{t=1}^{T}$ is strong mixing with mixing coefficients $\sum_{k>0}\alpha_k(\{e_t\}_{t\ge1})<M$.

\begin{lemma}
\label{sta_approx}
Under Assumption \ref{a5}, with a high probability $1-(\frac{M(\log T)^2}{2T})^{\frac{1}{3}}$,
\begin{equation}
\sup_x|\widetilde{F}_{T+1}(x)-F_e(x)|\le \frac{(\frac{M}{2})^{1/3}(\log T)^{2/3}}{T^{1/3}}.
\end{equation}
\end{lemma}
\begin{proof}
The proof follows the proof of Corollary $2$ in \citep{xu2023conformal}. Define \(v_T(x) := \sqrt{T}(\widetilde{F}_{T+1}(x) - F_e(x))\). Then, Proposition 7.1 in \citep{rio2017asymptotic} shows that
\begin{equation}
\mathbb{E}\left(\sup_x |v_T(x)|^2\right) \leq \left(1 + 4 \sum_{k=0}^T \alpha_k\right)\left(3 + \frac{\log T}{2 \log 2}\right)^2,
\end{equation}
where \(\alpha_k\) is the \(k\)th mixing coefficient. Since we assumed that the coefficients are summable with \(\sum_{k \geq 0} \alpha_k < M\) (for example, \(\alpha_k = \mathcal{O}(n^{-s}), s>1\)), Markov's Inequality shows that
\begin{equation}
\begin{aligned}
\mathbb{P}(\sup_x|\widetilde{F}_{T+1}(x)-F_e(x)| \geq s_T) & \leq \frac{\mathbb{E}(\sup_x |v_T(x)|^2 / T)}{s_T^2} \\
& \leq \frac{1 + 4 M}{T s_T^2}\left(3 + \frac{\log T}{2 \log 2}\right)^2.
\end{aligned}
\end{equation}

Thus, we let
\begin{equation}
s_T := \left(\frac{1 + 4 M}{T}\left(3 + \frac{\log T}{2 \log 2}\right)^2\right)^{1/3} \approx \left(\frac{M(\log T)^2}{2 T}\right)^{1/3},
\end{equation}
and see that
\begin{equation}
\begin{aligned}
& \mathbb{P}\left(\sup_x|\widetilde{F}_{T+1}(x)-F_e(x)| \leq \left(\frac{M(\log T)^2}{2 T}\right)^{1/3}\right) \\
& \geq 1 - \left(\frac{M(\log T)^2}{2 T}\right)^{1/3}.
\end{aligned}
\end{equation}

Hence, the event \(A_T\) is chosen so that conditioning on \(A_T\), 
\begin{equation}
\sup_x|\widetilde{F}_{T+1}(x)-F_e(x)| \leq \frac{(\frac{M}{2})^{1/3}(\log T)^{2/3}}{T^{1/3}}.    
\end{equation}
\end{proof}

\begin{corollary}
Assume $\{\varepsilon_t\}_{t=1}^{T}$ is a stationary and strong mixing sequence with mixing coefficient $0<\sum_{k>0}\alpha_k <M$. Under Assumption \ref{a2}, \ref{a3} and \ref{a4}, for any training size $T$ and $\alpha\in(0,1)$, we have
\begin{equation}
\begin{aligned}
&|\p(Y_{T+1} \in \widehat{C}_{T+1}^\alpha \mid X_{T+1}=x_{T+1})-(1-\alpha)| \\
&\qquad \leq 12 \frac{(\frac{M}{2})^{1/3}(\log T)^{2/3}}{T^{1/3}}+4(L_{T+1}+1)\left(\frac{\delta_T}{\sqrt{\lambda}}+\delta_T\right).
\end{aligned}
\end{equation}
\end{corollary}

When the true covariance matrix $\Sigma$ is known, lemma \ref{b5} also holds for the stationary and strong mixing process, and the proof can be directly used. Combining \ref{b5} and \ref{sta_approx} with the same technique in Theorem \ref{b8} yields the bound in Corollary \ref{c1}.

When the true covariance matrix $\Sigma$ is unknown, we only need to prove a similar result in Lemma \ref{bound2}. The difference is that we require the covariance estimator to converge to the true covariance matrix at a certain speed. As mentioned in Remark \ref{r1}, there is work presenting covariance estimators with guarantee in the stationary case, like \cite{chen2013covariance}. As long as we plug in certain estimators, the proof will follow, and the bound will depend on the guarantee of the estimator.

\end{document}